\documentclass[11pt,reqno,a4paper]{article}

\usepackage{amsmath}
\usepackage{amsthm}
\usepackage{amsfonts}
\usepackage{amssymb}
\usepackage{multirow}
\usepackage[dvips]{graphicx}
\usepackage[dvips]{color}
\usepackage{epsfig}
\usepackage{latexsym}
\usepackage{enumerate}
\usepackage{xspace}

\usepackage[bookmarksnumbered=true, citecolor=blue, colorlinks=false, hypertex]{hyperref}

\newtheorem{thm}{Theorem}[section]
\newtheorem{cor}[thm]{Corollary}
\newtheorem{lem}[thm]{Lemma}
\newtheorem{prop}[thm]{Proposition}

\newtheorem{defn}[thm]{Definition}

\numberwithin{equation}{section}

\newcommand{\limd}{\mathop{\lim}\limits}

\newcommand{\N}{\mathbb{N}}
\newcommand{\R}{\mathbb{R}}

\def \bN {\mathbb N}
\def \bZ {\mathbb Z}

\def \bR {\mathbb R}

\def \bC {\mathbb C}

\def \bt {{\bf t}}

\def \bv {{\bf v}}
\def \bx {{\bf x}}
\def \bz {{\bf z}}

\def \cB {{\cal B}}

\def \cD {{\cal D}}

\def \cH {{\cal H}}
\def \cI {{\cal I}}

\def \cL {{\cal L}}
\def \cM {{\cal M}}
\def \cS {{\cal S}}

\def \cP {{\cal P}}

\def \cE {{\cal E}}

\def \and {\, \mbox{\rm and}\, }
\def \sinc {\,{\rm sinc}\,}

\def \linearspan {\,{\rm span}\,}
\def \supp {\,{\rm supp}\,}

\newcommand{\spB}{\mathcal{B}}

\newcommand{\spI}{\mathcal{I}}
\newcommand{\spM}{\mathcal{M}}
\newcommand{\spS}{\mathcal{S}}

\newcommand{\vb}{\boldsymbol{b}}
\newcommand{\vc}{\boldsymbol{c}}

\newcommand{\vh}{\boldsymbol{h}}

\newcommand{\vj}{\boldsymbol{j}}

\newcommand{\vu}{\boldsymbol{u}}
\newcommand{\vs}{\boldsymbol{s}}
\newcommand{\vt}{\boldsymbol{t}}
\newcommand{\vzero}{\boldsymbol{0}}

\newcommand{\vxi}{\boldsymbol{\xi}}
\newcommand{\vy}{\boldsymbol{y}}

\newcommand{\me}{\text{e}}

\newcommand{\Addptsmalllabel}{(A4)\xspace}
\newcommand{\Hprimelabel}{(A3)\xspace}
\newcommand{\nonsingassum}{(A1)\xspace}
\DeclareMathOperator*{\argmin}{argmin}

\title{Reproducing Kernel Banach Spaces with the $\ell^1$ Norm\thanks{Supported
by Guangdong Provincial Government of China through the
``Computational Science Innovative Research Team" program.}}

\author{\quad Guohui Song\thanks{School of Mathematical and Statistical Sciences, Arizona State University, Tempe, AZ 85287, USA. E-mail address: {\it gsong9@asu.edu}.}, \quad Haizhang Zhang\thanks{School of Mathematics
and Computational Science and Guangdong Province Key Laboratory of Computational Science,
 Sun Yat-sen University, Guangzhou 510275, P. R. China. E-mail address: {\it zhhaizh2@sysu.edu.cn}.}, \quad and \quad Fred J. Hickernell\thanks{Department of Applied Mathematics, Illinois Institute of Technology, 10 W. 32$\text{nd}$ St., Chicago, IL, 60616, USA. E-mail address: {\it hickernell@iit.edu}. This author's work was supported in part by National Science Foundation grants DMS-0713848 and DMS-1115392.}}

\date{}

\begin{document}
\maketitle

\begin{abstract}
Targeting at sparse learning, we construct Banach spaces $\cB$ of functions on an input space $X$ with the following properties: (1) $\cB$ possesses an $\ell^1$ norm in the sense that $\cB$ is isometrically isomorphic to the Banach space of integrable functions on $X$ with respect to the counting measure; (2) point evaluations are continuous linear functionals on $\cB$ and are representable through a bilinear form with a kernel function; and (3) regularized learning schemes on $\cB$ satisfy the linear representer theorem. Examples of kernel functions admissible for the construction of such spaces are given.
\vspace*{0.2cm}

\end{abstract}

\vspace*{0.3cm}

\section{Introduction}
\setcounter{equation}{0}

It is now widely known that minimizing a loss function regularized by the $\ell^1$ norm yields sparsity in the resulting minimizer. The sparsity is essential for extracting relatively low dimensional features from sample data that usually live in a high dimensional space. When the square loss function is used in regression, the method is known as the lasso in statistics \cite{Tibshirani1996}. Recently, the methodology has been applied to compressive sensing where it is referred to as basis pursuit \cite{Cand`es2006,Chen1998}. The purpose of this paper is to establish an appropriate foundation for developing $\ell^1$ regularization for machine learning with reproducing kernels.

Past research on learning with kernels \cite{CuckerSmale2002,CuckerZhou2007,Evgeniou2000,Scholkopf2001a,Shawe-Taylor2004,Song2010,Vapnik1998} has mainly been built upon the theory of reproducing kernel Hilbert spaces (RKHS) \cite{Aronszajn1950}. There are many reasons that account for the success from such a choice. RKHS are by definition the Hilbert space of functions where point evaluations are continuous linear functionals. Sample data available for learning are usually modeled by point evaluations of the unknown target function. Therefore, RKHS is a class of function spaces where sampling is stable, a desirable feature in applications. By the Riesz representation theorem, continuous linear functionals on a Hilbert space are representable by the inner product on the space. This gives rise to the representation of point evaluation functionals on an RKHS by its associated reproducing kernel and leads to the celebrated representer theorem \cite{Kimeldorf1971} in machine learning.  This theorem states that the original minimization problem in a typically infinite dimensional RKHS can be converted into a problem of determining finitely many coefficients in a linear combination of the kernel function with one argument evaluated at the data sites.

For this representer theorem, the nonzero coefficients to be found are generally as many as the sampling points. For the sake of economy, it is hence desirable to regularize the class of candidate functions by some $\ell^1$ norm to force most of the coefficients to be zero. An attempt in this direction is the linear programming approach to coefficient based regularization for machine learning \cite{Scholkopf2001a}. The method lacks a general mathematical foundation like the RKHS though. In particular, it is unknown whether the algorithm results by some representer theorem from a minimization on an infinite dimensional Banach space. A consequence is that the hypothesis error in the learning rate estimate will not go away automatically as in the RKHS case \cite{Xiao2010}.

We aim at combining the reproducing kernel methods and the $\ell^1$ regularization technique. Specifically, we desire to construct function spaces with the following properties:
\begin{itemize}
\item[---] point evaluation functionals on the space are continuous and can be represented by some kernel function;

\item[---] the space possesses an $\ell^1$ norm;

\item[---] a linear representer theorem holds for regularized learning schemes on the space.
\end{itemize}
There are three ways of representing continuous point evaluation functionals in a function space: by an inner product, by a semi-inner product \cite{Giles1967,Lumer1961}, or by a bilinear form on the tensor product of the space and its dual space. Since the space we constructed is expected to have an $\ell^1$ norm, it can not have an inner product. Semi-inner products are a natural substitute for inner products in Banach spaces. A notion of reproducing kernel Banach spaces (RKBS) was established in \cite{Zhang2009,Zhangjogo} via the semi-inner product. The spaces considered there are uniformly convex and uniformly Fr\'{e}chet differentiable to ensure that continuous linear functionals have a unique representation by the semi-inner product. An infinite dimensional Banach space with the $\ell^1$ norm is non-reflexive. As a consequence, there is no guarantee \cite{James1963/1964} that the semi-inner product is able to represent all continuous point evaluation functionals in such a space. For these reasons, we shall pursue the third approach in this study, that is, to represent the point evaluation functionals by a bilinear form. We briefly introduce the construction and main results of the paper below.

Let $X$ be a prescribed set that we call the input space. The construction starts directly with a complex-valued function $K$ on $X\times X$, which is not necessarily Hermitian. For the constructed space to have the three desirable properties described above, $K$ needs to be an admissible kernel. To introduce this class of functions crucial to our construction, we denote for any set $\Omega$ by $\ell^1(\Omega)$ the Banach space of functions on $\Omega$ that is integrable with respect to the counting measure on $\Omega$. In other words,
$$
\ell^1(\Omega):=\{\vc=(c_t\in\bC:t\in\Omega):\|\vc \|_{\ell^1(\Omega)}:=\sum_{t\in\Omega}|c_t|<+\infty\}.
$$
Note that $\Omega$ might be uncountable but for every $\vc \in \ell^1(\Omega)$, $\supp \vc:=\{t\in\Omega:c_t\ne0\}$ must be countable. Finally,
we define the set $\bN_n:=\{1,2,\ldots,n\}$ for all $n\in\bN$.

\begin{defn} \label{admisskerdef} A function $K$ on $X\times X$ is called an \emph{admissible kernel} for the construction of RKBS on $X$ with the $\ell^1$ norm if the following requirements are satisfied:
\begin{enumerate}
\renewcommand{\labelenumi}{(A\arabic{enumi})}
\item for all sequences $\bx=\{{x}_j:j\in\bN_n\}\subseteq X$ of pairwise distinct sampling points, the matrix
\begin{equation}\label{kernelkx}
K[\bx]:=[K(x_k,x_j): \ j,k\in\bN_n] \in \bC^{n\times n}
\end{equation}
is nonsingular,
\item $K$ is bounded, namely, $|K(s,t)|\le M$ for some positive constant $M$ and all $s,t\in X$,
\item \label{Hprime}
for all pairwise distinct ${x}_j\in X$, $j\in\bN$ and $\vc\in\ell^1(\bN)$, $\sum_{j=1}^\infty c_j K({x}_j, {x})=0$ for all ${x}\in X$ implies $\vc=0$, and
\item for all pairwise distinct $x_1, x_2, \ldots, x_{n+1} \in X$,
\begin{equation}\label{H2condition}
\left\| (K[\bx])^{-1} K_{\bx}({x}_{n+1})\right\|_{\ell^1(\bN_n)} \leq 1,
\end{equation}
where $K_\bx(x)=(K(x,x_j):j\in\bN_n)^T\in\bC^n$.
\end{enumerate}
\end{defn}

The following theorem will be proved in the next three sections.

\begin{thm}\label{construction}
If $K$ is an admissible kernel on $X\times X$ then
\begin{equation}\label{cbl1norm}
\cB:=\biggl\{\sum_{t\in \supp \vc}\vc_{t} K(t,\cdot): \vc\in\ell^1(X)\biggr\}\mbox{ with the norm }\biggl\|\sum_{t\in \supp \vc}c_{t} K(t,\cdot)\biggr\|_\cB:=\|\vc\|_{\ell^1(X)}
\end{equation}
and $\cB^\sharp$, the completion of the vector space of functions $\sum_{j=1}^n c_j K(\cdot,{x}_j)$, $x_j\in X$ under the supremum norm
$$
\biggl\|\sum_{j=1}^n c_j K(\cdot,{x}_j)\biggr\|_{\cB^\sharp}:=\sup\biggl\{\biggl|\sum_{j=1}^n c_j K(x,{x}_j)\biggr|:x\in X\biggr\},
$$
are both Banach spaces of functions on $X$ where point evaluations are continuous linear functionals. In addition, the bilinear form
\begin{equation}\label{bilinearform}
\biggl\langle \sum_{j=1}^na_j K({s}_j,\cdot),\sum_{k=1}^m b_k K(\cdot,t_k)\biggr\rangle_K:=\sum_{j=1}^n\sum_{k=1}^ma_j b_k K(s_j,t_k),\ \ s_j,t_k\in X
\end{equation}
can be extended to $\cB\times\cB^\sharp$ such that
$$
|\langle f,g\rangle_K|\le \|f\|_\cB\|g\|_{\cB^\sharp}\mbox{ for all }f\in\cB,\ g\in\cB^\sharp
$$
and
$$
\langle f, K(\cdot,{x})\rangle_K=f({x}),\ \ \langle K({x},\cdot),g\rangle_K=g({x})\mbox{ for all }{x}\in X, f\in\cB,\ g\in\cB^\sharp.
$$
Furthermore, for every regularized learning scheme of the form
$$
\inf_{f\in\cB}V(f(x_1),f(x_2),\cdots,f(x_n))+\mu \phi(\|f\|_\cB),
$$
where $\mu$ is a positive regularization parameter, $V$ and $\phi$ are nonnegative continuous functions with $\lim_{t\to\infty}\phi(t)=+\infty$, there exists a minimizer, $f_0$, of the form
$$
f_0(x)=\sum_{j=1}^n c_j K(x_j,x),\ \ x\in X
$$
for some coefficients $c_j\in \bC$, $j\in\bN_n$.

Conversely, for the constructed spaces $\cB$ and $\cB^\sharp$ to enjoy those desirable properties, $K$ must be an admissible kernel on $X\times X$.
\end{thm}

The organization of the paper is as follows. We first present a general construction of Banach spaces of functions with a reproducing kernel in the next section. In Section 3, we specify the construction to the building of RKBS with the $\ell^1$ norm as described in Theorem \ref{construction}. In Section \ref{repthemsec}, we study the conditions on the reproducing kernel so that regularized learning schemes on the constructed spaces satisfy the linear representer theorem. In the last section, we show that the Brownian bridge kernel and the exponential kernel are admissible kernels.  In the final section, condition (A4), the most stringent condition in Definition \ref{admisskerdef} is relaxed, which leads to a modified version of Theorem \ref{construction}.

\section{A General Construction} \label{gencostsec}
\setcounter{equation}{0}

To ensure that there exists a reproducing kernel, we shall start the construction of the Banach space based on such a function. Let $X$ be an input space and let $K$ be a function on $X\times X$. Introduce the vector space
\begin{equation*}
\spB_0: = \linearspan\{K({x}, \cdot): {x}\in {X}\}.
\end{equation*}
 Note that unlike reproducing kernels for Hilbert spaces, this $K$ is not necessarily symmetric in its arguments or positive definite.  Suppose that a norm $\|\cdot \|_{\spB_0}$ is imposed on $\spB_0$ such that point evaluation functionals are continuous on $\spB_0$. That is, for any ${x}\in X$, there exists a positive constant $M_{{x}}$ such that
\begin{equation}\label{pointevaluaitononcb0}
|\delta_x(f)|=|f({x})| \leq M_{{x}} \| f \|_{\spB_0}\mbox{ for all }f\in \spB_0.
\end{equation}
The function $K$ and the norm on $\cB_0$ will be explicitly given in a specific construction.

In \cite{Zhang2009,Zhangacha,Zhangjogo}, a vector space $\cB$ is called an RKBS on $X$ if it is a uniformly convex and uniformly Fr\'{e}chet differentiable Banach space of functions on $X$ and point evaluation functionals are continuous on $\cB$. The uniform convexity and uniform Fr\'{e}chet differentiability were imposed there to ensure the existence of a reproducing kernel for representing the point evaluation functionals. By the results to be established in the current paper, these stronger conditions are not necessary. To accommodate the search for alternatives, we introduce the following definitions.

\begin{defn}
 The space $\cB$ is called a \emph{Banach space of functions} if the point evaluation functionals are consistent with the norm on $\cB$ in the sense that for all $f\in\cB$, $\|f\|_\cB=0$ if and only if $f$ vanishes everywhere on $X$. A Banach space $\cB$ of functions on $X$ is said to be a \emph{pre-RKBS} on $X$ if point evaluations are continuous linear functionals on $\cB$.
\end{defn}

We plan to complete $\spB_0$ by the norm $\|\cdot \|_{\spB_0}$ to obtain a pre-RKBS $\spB$. Two things need to be checked for the approach to succeed. An abstract completion of $\cB_0$ might not consist of functions, or might not have bounded point evaluation functionals. We shall present a Banach completion process that yields a space of functions. Let $\{f_n:n\in\bN\}$ be a Cauchy sequence in $\spB_0$. Since point evaluation functionals are continuous on $\spB_0$, for any ${x}\in {X}$, the sequence $\{f_n({x}):n\in\bN\}$ converges in $\bC$. We denote the limit by $f({x})$, which defines a function on $X$. One sees that two equivalent Cauchy sequences in $\cB_0$ give the same function. We let $\spB$ be composed of all such limit functions with the norm $\|f\|_\spB: = \lim_{n\to \infty} \|f_n\|_{\spB_0}$.

To investigate conditions for $\cB$ to be a pre-RKBS, we need to invoke the following assumption.
\begin{defn} A normed vector space $V$ of functions on $X$ satisfies the \emph{Norm Consistency Property} if for every Cauchy sequence $\{f_n: n\in\bN\}$ in $V$, $\limd_{n\rightarrow \infty} f_n({x})=0$ for all ${x}\in {X}$ implies $\limd_{n\rightarrow \infty}\| f_n\|_{V}=0$.
\end{defn}

\begin{prop}\label{BanachCompletion}
The norm $\|\cdot\|_\cB$ is well-defined and makes $\cB$ a pre-RKBS on $X$ if and only if $\cB_0$ satisfies the Norm Consistency Property.
\end{prop}
\begin{proof}
We first show the necessity. If $\spB$ is a Banach space then $\|\cdot \|_\spB$ is a well-defined norm. The validity of the Norm Consistency Property follows directly from $\|0 \|_\spB=0$.

We next prove the sufficiency. Suppose that the Norm Consistency Property holds for $\cB_0$. We first show that $\| \cdot \|_\spB$ is a well-defined norm. Suppose that $\{f_n: n\in\bN\}$ and $\{g_n: n\in\bN\}$ are both Cauchy sequences in $\spB_0$ such that $\limd_{n\rightarrow \infty} f_n({x}) =\limd_{n\rightarrow \infty} g_n({x})$ for all ${x}\in X$. We need to show that $\limd_{n\rightarrow \infty} \|f_n \|_{\spB_0} =  \limd_{n\rightarrow \infty} \|g_n \|_{\spB_0}$. Clearly, $f_n - g_n$ forms a Cauchy sequence in $\cB_0$. Since $\limd_{n\to \infty}( f_n - g_n )({x}) =0$ for all ${x}\in {X}$, it follows from the Norm Consistency Property that $\limd_{n\rightarrow \infty}\| f_n - g_n\|_{\spB_0}=0$, which implies $\limd_{n\rightarrow \infty} \|f_n \|_{\spB_0} =  \limd_{n\rightarrow \infty} \|g_n \|_{\spB_0}$. Therefore, $\|\cdot\|_\cB$ is well-defined. As a result, $\cB$ is isometrically isomorphic to the abstract Banach space that is the completion of $\cB_0$. It implies that $\cB$ is a Banach space and $\cB_0$ is dense in $\cB$. Moreover, it follows immediately from the Norm Consistency Property that $\cB$ is a Banach space of functions. It remains to show that the point evaluation functional $\delta_{{x}}$ is continuous on $\spB$ for all ${x}\in {X}$. Let ${x}\in X$ and $f\in \spB$. By definition, there exists a Cauchy sequence $\{ f_n:n\in\bN\}$ in $\spB_0$ such that
\begin{equation*}
f({x}) = \limd_{n\rightarrow \infty} f_n({x})\mbox{ for all }{x}\in {X}, \quad {\rm and}\quad \| f\|_\spB = \limd_{n\rightarrow \infty} \| f_n\|_{\spB_0}.
\end{equation*}
Since $\delta_{{x}}$ is continuous on $\spB_0$, there exists a positive constant $M_{{x}}$ such that
\begin{equation*}
|f_n({x})| \leq M_{{x}} \|f_n \|_{\spB_0} \quad {\rm for~ all~} n\in \N.
\end{equation*}
Taking the limits on both sides, we have $|f({x})| \leq  M_x \|f \|_\spB$. The proof is complete.
\end{proof}

In the rest of this section, we assume the Norm Consistency Property for $\cB_0$ and aim at deriving a reproducing kernel for $\spB$. To this end, we set
$$
\cB_0^\sharp:=\linearspan\{K(\cdot,x):x\in X\}
$$
and
define a bilinear form $\langle\cdot, \cdot\rangle_K$ on $\spB_0 \times \spB_0^\sharp$ by (\ref{bilinearform}). It is straightforward to observe that
\begin{equation*}
\langle f, K(\cdot,{x})\rangle_K = f({x}),\quad \langle K(x,\cdot),g\rangle_K=g(x) \mbox{ for all } f\in \spB_0,\ g\in\cB_0^\sharp\mbox{ and } {x}\in {X}.
\end{equation*}
It means (\ref{bilinearform}) is well defined and that $K$ is able to reproduce the point evaluations of functions on $\cB_0$ via this bilinear form. We need to extend this property to the whole space $\cB$ in order to claim that it is a reproducing kernel for $\cB$. For this purpose, we define another norm
\begin{equation}\label{sharpnorm}
\|g\|_{\spB_0^\sharp} := \sup_{f\in \spB_0, f\neq 0} \frac{\left|\langle f,g\rangle_K \right|}{\|f\|_{\spB_0}}, \quad g\in \spB_0^\sharp.
\end{equation}
The next result indicates that the above norm is well-defined.

\begin{prop}\label{boundnorm}
The norm $\|\cdot\|_{\spB_0^\sharp}$ is well-defined and point evaluation functionals are continuous on $\cB_0^\sharp$ if and only if point evaluation functionals are continuous on $\spB_0$.
\end{prop}
\begin{proof}
We begin with the sufficiency. Suppose that point evaluation functionals are continuous on $\spB_0$. That is, for any ${x}\in {X}$ there exists a positive constant $M_{{x}}$ satisfying (\ref{pointevaluaitononcb0}). Let $g\in \spB_0^\sharp$. It must be of the form $g=\sum_{j=1}^n a_j K(\cdot,{x}_j)$ for some $a_j \in \bC$ and ${x}_j\in X$, $j\in\bN_n$, $n\in \N$. We have for all $f\in \spB_0$
\begin{eqnarray*}
\frac{|\langle f,g\rangle_K|}{\|f\|_{\spB_0}} = \frac{|\langle f,\sum_{j=1}^n a_j K(\cdot,{x}_j)\rangle_K|}{\|f\|_{\spB_0}}
= \frac{\left|\sum_{j=1}^n a_j f({x}_j)\right|}{\|f\|_{\spB_0}} \leq \sum_{j=1}^n |a_j| M_{{x}_j},
\end{eqnarray*}
which implies that $\|g\|_{\spB_0^\sharp}$ is well-defined. We next prove that point evaluation functionals are continuous on $\spB_0^\sharp$. By \eqref{sharpnorm}, we have for all $ f\in\cB_0, g\in \spB_0^\sharp$
\begin{equation}\label{boundnorm1}
\left|\langle f,g\rangle_K \right| \leq  \|f\|_{\spB_0} \|g\|_{\spB_0^\sharp}.
\end{equation}
For any ${x}\in {X}$, taking $f=K({x},\cdot)$ in the above inequality yields that
\begin{equation*}
|g({x})|=|\langle K(x,\cdot),g\rangle_K| \leq \|K(x,\cdot)\|_{\spB_0} \|g\|_{\spB_0^\sharp} \quad {\rm for~ all~} g\in \spB_0^\sharp.
\end{equation*}
It follows that the point evaluation functional $\delta_{{x}}$ is continuous on $\spB_0^\sharp$ as $\|K({x},\cdot)\|_{\spB_0}$ is a constant independent of $g$.

We next turn to the necessity. Suppose $\|g\|_{\spB_0^\sharp}$ is well-defined for all $g\in \spB_0^\sharp$.
For any ${x}\in {X}$, letting $g=K(\cdot,{x})$ in \eqref{boundnorm1} yields
\begin{equation*}
|f({x})| \leq \|K(\cdot,{x})\|_{\spB_0^\sharp} \|f\|_{\spB_0},
\end{equation*}
which implies that point evaluation functionals are continuous on $\spB_0$.
\end{proof}

We complete $\spB_0^\sharp$ using the norm $\|\cdot\|_{\spB_0^\sharp}$ to a Banach space $\spB^{\sharp}$ by the process described before Proposition \ref{BanachCompletion}. We have the following observation similar to that about the space $\cB$.

\begin{prop}\label{BsharpRKBS}
The space $\cB^\sharp$ is a pre-RKBS on $X$ if and only if the normed vector space $\spB_0^\sharp$ satisfies the Norm Consistency Property.
\end{prop}

In the following discussion, suppose that $\cB_0^\sharp$ endowed with the norm $\|\cdot\|_{\cB_0^\sharp}$ has the Norm Consistency Property. By applying the Hahn-Banach extension theorem twice, we can extend the bilinear form $\langle\cdot, \cdot\rangle_K$ from $\spB_0\times \spB_0^\sharp$ to $\spB\times \spB^\sharp$ in a unique way such that
\begin{equation}\label{bilinear}
\left|\langle f,g\rangle_K\right| \leq \| f\|_\spB \| g \|_{\spB^\sharp}, \quad f\in \spB,\  g\in \spB^\sharp.
\end{equation}
The next result tells that the definition of $\|\cdot \|_{\spB_0^\sharp}$ in (\ref{sharpnorm}) can be extended to $\spB^\sharp$.

\begin{prop}\label{sharpcompletion}
Suppose that point evaluation functionals are continuous on $\spB_0$. If both $\spB_0$ and $\spB_0^\sharp$ satisfy the Norm Consistency Property then we have
\begin{equation}\label{isometricembedding}
\|g\|_{\spB^\sharp} = \sup_{f\in \spB, f\neq 0} \frac{\left|\langle f,g\rangle_K \right|}{\|f\|_\spB}, \quad g\in \spB^\sharp.
\end{equation}
\end{prop}
\begin{proof}
By \eqref{bilinear}, the right hand side above is bounded by the left hand side. We only need to prove the other direction of the inequality. We first show it for functions in $\spB_0^\sharp$. Let $g\in \spB_0^\sharp$. It is straightforward to observe that
\begin{equation}\label{bilinear1}
\|g\|_{\spB^\sharp} = \sup_{f\in \spB_0, f\neq 0} \frac{\left|\langle f,g\rangle_K \right|}{\|f\|_\spB}\leq \sup_{f\in \spB, f\neq 0} \frac{\left|\langle f,g\rangle_K \right|}{\|f\|_\spB}.
\end{equation}
Now let $g$ be an arbitrary but fixed function in $\cB^\sharp$. Since $\spB_0^\sharp$ is dense in $\spB^\sharp$, there exists $\{g_n:n\in\bN\} \subseteq \spB_0^\sharp$ such that $\|g-g_n\|_{\spB^\sharp}\rightarrow 0$ as $n\rightarrow \infty$. This together with \eqref{bilinear1} implies
\begin{equation*}
\|g\|_{\spB^\sharp} = \lim_{n\rightarrow \infty}\|g_n\|_{\spB^\sharp} \leq \lim_{n\rightarrow \infty} \sup_{f\in \spB, f\neq 0} \frac{\left|\langle f,g_n\rangle_K \right|}{\|f\|_\spB}.
\end{equation*}
Note that
\begin{equation*}
\frac{\left|\langle f,g_n\rangle_K \right|}{\|f\|_\spB} \leq \frac{\left|\langle f,g\rangle_K \right|}{\|f\|_\spB} + \frac{\left|\langle f,g-g_n\rangle_K \right|}{\|f\|_\spB} \leq \frac{\left|\langle f,g\rangle_K \right|}{\|f\|_\spB} + \|g-g_n\|_{\spB^\sharp}.
\end{equation*}
It follows from the above two equations that
\begin{equation*}
\|g\|_{\spB^\sharp} \leq \lim_{n\rightarrow \infty} \sup_{f\in \spB, f\neq 0} \left[ \frac{\left|\langle f,g\rangle_K \right|}{\|f\|_\spB} + \|g-g_n\|_{\spB^\sharp}\right] = \sup_{f\in \spB, f\neq 0}\frac{\left|\langle f,g\rangle_K \right|}{\|f\|_\spB},
\end{equation*}
which completes the proof.
\end{proof}

We next present necessary and sufficient conditions for $K$ to be able to reproduce point evaluation functionals on $\cB$ and $\cB^\sharp$ by the bilinear form. We shall see that assuming the Norm Consistency Property, both $\cB$ and $\cB^\sharp$ are Banach spaces of functions on $X$ such that the point evaluation functionals are continuous and can be represented by the bilinear form with the function $K$. It is in this sense that $\cB$ and $\cB^\sharp$ are said to be a {\it reproducing kernel Banach space with the reproducing kernel} $K$.

\begin{thm}\label{reproducing}
Suppose that $\spB_0$ and $\spB_0^\sharp$ satisfy the Norm Consistency Property. Then both $\cB$ and $\cB^\sharp$ are pre-RKBS on $X$ and the kernel $K$ reproduces function values via the bilinear form, namely,
\begin{equation}\label{reproduce1}
\langle f, K(\cdot,{x})\rangle_K = f({x})\mbox{ for all }{x}\in X\mbox{ and }f\in\cB
\end{equation}
and
\begin{equation}\label{reproduce2}
\langle K({x},\cdot), g\rangle_K = g({x})\mbox{ for all }{x}\in X\mbox{ and }g\in\cB^\sharp.
\end{equation}
Thus, $\spB$ and $\spB^\sharp$ are reproducing kernel Banach spaces (RKBS).
\end{thm}
\begin{proof}
By Propositions \ref{BanachCompletion} and \ref{BsharpRKBS}, both $\cB$ and $\cB^\sharp$ are pre-RKBS on $X$. For each $f\in \spB$, there exists a sequence $\{f_n:n\in\bN\} \subseteq \spB_0$ convergent to $f$. As a consequence, we have for any ${x}\in {X}$
\begin{equation*}
f({x})=\lim_{n\rightarrow \infty}f_n({x})=\lim_{n\rightarrow \infty} \langle f_n, K(\cdot,{x})\rangle_K.
\end{equation*}
By \eqref{bilinear}, $\langle\cdot, K(\cdot,{x})\rangle_K$ is a bounded linear functional on $\spB$, which implies
\begin{equation*}
\lim_{n\rightarrow \infty} \langle f_n, K(\cdot,{x})\rangle_K = \langle f, K(\cdot,{x})\rangle_K.
\end{equation*}
Combining the above two equations proves (\ref{reproduce1}). Equation (\ref{reproduce2}) can be proved similarly.
\end{proof}

We next discuss the relationship between the space $\spB^\sharp$ and the dual space $\spB^*$ of $\spB$. It is clear by \eqref{bilinear} and \eqref{isometricembedding} that the mapping $\cL$ from $\cB^\sharp$ to $\cB^*$ defined by the bilinear form,
\begin{equation}\label{embedding}
(\cL g)(f):=\langle f,g\rangle_K,\ \ f\in\cB,\ \ g\in\cB^\sharp,
\end{equation}
is isometric and linear. In other words, $\cL$ is an embedding from $\cB^\sharp$ to $\cB^*$. We next present a necessary and sufficient condition for it to be surjective.

\begin{prop}
Suppose that both $\spB_0$ and $\spB_0^\sharp$ satisfy the Norm Consistency Property. The mapping $\cL$ defined by (\ref{embedding}) is surjective onto $\cB^*$ if and only if for any proper closed subspace $\spM\subsetneqq \spB$, the orthogonal space $\spM^\perp:= \{ g\in \spB^\sharp: \langle f,g\rangle_K=0\mbox{ for all }f\in\cM\}$ is nontrivial.
\end{prop}
\begin{proof}
We first prove the necessity. For any proper closed subspace $\spM\subsetneqq \spB$, by the Hahn-Banach theorem, there exists a nontrivial functional $\nu \in \spB^*$ such that $\nu(f)=0$ for all $f\in \spM$. If $\cL$ is surjective then there exists a function $g\in \spB^\sharp$ such that $\cL(g)=\nu$, namely, $\nu(f)=\langle f,g\rangle_K$ for all $f\in \spB$. It follows that $g\in \spM^\perp$ and $g\neq 0$ as $\nu$ is nontrivial.

We next show the sufficiency. Let $\nu$ be a nontrivial functional in $\spB^*$. Then its kernel $\ker(\nu)$ is a proper closed subspace of $\cB$. By assumption, there exists a nonzero function $g\in \spM^\perp$. This enables us to find a function $f_0\in \spB \backslash \spM$ such that $\langle f_0, g\rangle_K\neq 0$ and $\nu(f_0)=1$. Set $g_0: =g/{\langle f_0,g\rangle_K}$. Since $f-\nu(f)f_0\in\ker(\nu)$ for all $f\in\cB$, we get for any $f\in \spM$
\begin{eqnarray*}
\langle f,g_0\rangle_K=\langle f-\nu(f)f_0, g_0\rangle_K + \langle\nu(f)f_0, g_0\rangle_K=\nu(f)\langle f_0,g_0\rangle_K=\nu(f),
\end{eqnarray*}
which implies that $\cL$ is surjective.
\end{proof}

We close the section with a conclusion on the general construction and the related results presented above.
\begin{thm}\label{procedures}
Suppose that
\begin{enumerate}
\renewcommand{\labelenumi} {(\alph{enumi})}
\item the vector space $\cB_0=\linearspan\{K({x},\cdot): {x}\in X\}$ with the norm $\|\cdot\|_{\cB_0}$  has the Norm Consistency Property, and
\item point evaluation functionals are continuous on $\cB_0$.
\end{enumerate}
Then the following statements hold true:
\begin{enumerate}
\renewcommand{\labelenumi}{(\arabic{enumi})}
\item  $\cB_0$ can be completed to a pre-RKBS $\cB$ on $X$;

\item the norm $\|\cdot\|_{\cB_0^\sharp}$ given by (\ref{sharpnorm}) is well-defined and point evaluation functionals are bounded on $\cB_0^\sharp$ with respect to this norm;

\item if $\cB_0^\sharp$ satisfies the Norm Consistency Property as well then $\cB_0^\sharp$ can be completed to an RKBS $\cB^\sharp$ and $K$ is the reproducing kernel for both $\cB$ and $\cB^\sharp$ in the sense that (\ref{reproduce1}) and (\ref{reproduce2}) hold true. In this case, $\cB^\sharp$ can be isometrically embedded into $\cB^*$ via the bilinear form, and the embedding is surjective if and only if for any proper closed subspace $\spM$ of $\spB$, $\spM^\perp$ is nontrivial.
\end{enumerate}
\end{thm}

\section{RKBS with the $\ell^1$ Norm} \label{RKBSlonesec}
We shall follow the procedures in Theorem \ref{procedures} to construct an RKBS with the $\ell^1$ norm in this section. To start, we let $K$ be a bounded function on $X\times X$ such that
\begin{equation}\label{linearlyindependentK}
K({x}_j,\cdot),j\in\bN_n\mbox{ are linearly independent for all pairwise distinct points } {x}_j\in X,j\in\bN_n.
\end{equation}
Note that this assumption is implied by Admissibility Assumption \nonsingassum, but is somewhat weaker than \nonsingassum. Introduce an $\ell^1$ norm on $\cB_0=\linearspan\{K({x},\cdot): {x}\in X\}$ by setting for all finitely many pairwise distinct points ${x}_j\in X$ and constants $c_j\in\bC$, $j\in\bN_m$, $m\in\bN$
\begin{equation}\label{l1Banach}
\biggl\| \sum_{j=1}^m c_jK({x}_j,\cdot) \biggr\|_{\spB_0}: = \sum_{j=1}^m|c_j|.
\end{equation}
	Since $K$ is bounded, it is clear that point evaluation functionals are bounded on $\cB_0$. We next check the important Norm Consistency Property and find that it is implied by the Admissibility Assumption above.

\begin{prop}\label{H1H2}
The space $\cB_0$ with the norm (\ref{l1Banach}) satisfies the Norm Consistency Property if and only if $K$ satisfies \Hprimelabel.
\end{prop}
\begin{proof}
We first show the necessity. Suppose that for some $\vc\in\ell^1(\bN)$ and pairwise distinct $\{{x}_j\in X:j\in\bN\}$,  $\sum_{j=1}^\infty c_jK({x}_j, {x})=0$ for all ${x}\in X$. Let $f_n:=\sum_{j=1}^n c_j K({x}_j,\cdot)$ for all $n\in \N$. Since $\vc \in \ell^1(\bN)$, $\{f_n: n\in \bN\}$ forms a Cauchy sequence in $\cB_0$. Moreover, $\limd_{n\rightarrow \infty}f_n({x})=0$ for all ${x}\in {X}$ as $K$ is bounded on $X\times X$. It follows from the Norm Consistency Property that $\limd_{n\rightarrow \infty}\| f_n\|_{\spB_0} = \limd_{n\rightarrow \infty} \sum_{j=1}^n |c_j| = \| \vc\|_{\ell^1(\bN)}=0$. Therefore, \Hprimelabel holds true.

On the other hand, suppose that $K$ satisfies \Hprimelabel. Let $\{f_n:n\in\bN\}$ be a Cauchy sequence in $\spB_0$ with $\limd_{n\rightarrow \infty}f_n({x})=0$ for all ${x}\in {X}$. We can find pairwise distinct ${x}_j\in X$, $j\in\bN$ such that for any $n\in \N$
\begin{equation*}
f_n = \sum_{j=1}^\infty c_{n,j}K({x}_j,\cdot),
\end{equation*}
where $\vc_n:=(c_{n,j}: j\in \bN)$ has finitely many nonzero components. By definition (\ref{l1Banach}), $\{\vc_n:n\in\bN\}$ is a Cauchy sequence in $\ell^1(\bN)$. Let $\vc$ be its limit in $\ell^1(\bN)$ and define
\begin{equation*}
f: = \sum_{j=1}^\infty c_{j}K(x_j,\cdot).
\end{equation*}
Suppose that $|K(s,t)|\le M$ for some positive constant $M$ and all $s, t \in X$. A direct calculation gives that for any ${x}\in {X}$
\begin{equation*}
|f_n({x}) - f({x})| = \biggl|\sum_{j=1}^\infty (c_{n,j} - c_j) K({x}_j,{x}) \biggr| \leq M \|\vc_n - \vc \|_{\ell^1(\bN)}.
\end{equation*}
It follows that $\limd_{n\rightarrow \infty}f_n({x})=f({x})$ for all ${x}\in {X}$. Since $\limd_{n\rightarrow \infty}f_n({x})=0$ for all ${x}\in {X}$, we have $f({x})=0$ for all ${x}\in {X}$. By \Hprimelabel, $\vc=0$, which implies
\begin{equation*}
\limd_{n\rightarrow \infty}\| f_n\|_{\spB_0}=\limd_{n\rightarrow \infty}\|\vc_n \|_{\ell^1(\bN)} =\| \vc\|_{\ell^1(\bN)}=0.
\end{equation*}
The proof is complete.
\end{proof}

Functions $K$ satisfying property \Hprimelabel will be given later. We assume for the time being that \Hprimelabel holds true. One sees from the proof of Proposition \ref{H1H2} that $\cB$ has the form (\ref{cbl1norm}). We remark that in the preparation of the paper, we came across a Banach space with a form similar to (\ref{cbl1norm}) used in \cite{Xiao2010} for error estimates with linear programming regularization. One observes from (\ref{cbl1norm}) that $\ell^1(X)$ is isometrically isomorphic to $\cB$ through the mapping
$$
\Phi(\vc):=\sum_{t\in X}\vc_{t} K(t,\cdot),\ \ \vc\in\ell^1(X).
$$
In this sense, we say that $\cB$ is a pre-RKBS on $X$ with the $\ell^1$ norm. It remains to derive a reproducing kernel for it. By Theorem \ref{reproducing}, it suffices to check the Norm Consistency Property for $\cB_0^{\sharp}$. We shall show that the Norm Consistency Property automatically holds true for $\cB_0^\sharp$ without any additional requirement. To this end, we first calculate a specific form of the norm $\| \cdot\|_{\spB_0^\sharp}$.

Denote for any function $g$ on $X$ by $\|g\|_{L^\infty(X)}$ the supremum of $|g({x})|$ over ${x}\in X$.

\begin{lem}\label{BsharpNorm}
There holds for any function $g\in \spB_0^\sharp$ that $\| g \|_{\spB_0^\sharp}=\| g \|_{L^\infty(X)}$.
\end{lem}
\begin{proof}
We first prove that $\| g \|_{\spB_0^\sharp}$ is bounded by $\| g \|_{L^\infty(X)}$. Any $f\in \spB_0$ has the form $f=\sum_{j=1}^n c_j K({x}_j,\cdot)$ for some $c_j \in \bC$ and pairwise distinct ${x}_j\in X$, $j\in\bN_n$. We verify that
\begin{equation*}
\left| \langle f, g\rangle_K \right| =\biggl| \biggl\langle\sum_{j=1}^n c_j K({x}_j,\cdot), g \biggr\rangle \biggr|= \biggl|\sum_{j=1}^n c_j g({x}_j) \biggr| \leq \|g \|_{L^\infty(X)}\sum_{j=1}^n |c_j| = \|g \|_{L^\infty(X)}\| f\|_{\spB_0}  ,
\end{equation*}
which implies $\| g \|_{\spB_0^\sharp} \leq \| g \|_{L^\infty(X)}$. For the other direction, we notice for all ${x}_0\in X$
\begin{equation*}
\| g \|_{\spB_0^\sharp}\ge \frac{\left|\langle K({x}_0,\cdot),g\rangle_K \right|}{\|K({x}_0,\cdot)\|_{\spB_0}} = | g({x}_0)|.
\end{equation*}
Since ${x}_0$ is arbitrarily chosen, we have $\| g \|_{\spB_0^\sharp} \geq \| g \|_{L^\infty(X)}$.
\end{proof}

We show that the space $\spB^\sharp$ is also a pre-RKBS on $X$.

\begin{lem}\label{bsharph1ok}
The space $\cB_0^\sharp$ satisfies the Norm Consistency Property.
\end{lem}
\begin{proof}
Let $\{f_n:n\in\bN\}$ be a Cauchy sequence in $\spB_0^\sharp$ with $\limd_{n\rightarrow \infty} f_n({x})=0$ for all ${x}\in {X}$. By Lemma \ref{BsharpNorm}, there exists for any $\epsilon>0$ some positive integer $N_0$ such that when $m,n \geq N_0$,
$$
| f_m({x}) -f_n({x}) | \leq  \epsilon \quad \mbox{for all } {x}\in {X}.
$$
Since $\limd_{n\rightarrow \infty} f_n({x})=0$, we let $n$ goes to infinity in the above inequality to obtain that when $m\geq N_0$,
\begin{equation*}
| f_m({x})| \leq  \epsilon \quad \mbox{for all } {x}\in {X}.
\end{equation*}
In other words, $\|f_m\|_{L^\infty(X)} \leq  \epsilon$ when $m\geq N_0$, implying $\limd_{n\rightarrow \infty}\| f_n\|_{L^\infty(X)}=0$.
\end{proof}

By Proposition \ref{H1H2} and Lemmas \ref{BsharpNorm} and \ref{bsharph1ok}, we conclude our construction of RKBS with the $\ell^1$ norm in the following result.

\begin{thm}\label{H1condition}
Let $K$ be a bounded function on $X\times X$ that satisfies \Hprimelabel. Then $\cB$ having the form (\ref{cbl1norm}) and $\cB^\sharp$ are RKBS on $X$ with the reproducing kernel $K$.
\end{thm}

We shall discuss in the rest of this section conditions on translation invariant   $K : \bR^d\times\bR^d \to \bC$ for which Admissibility Assumption \Hprimelabel holds. Specifically, such $K$ are of the form
\begin{equation}\label{translationinvariantK}
K(\vs, \vt)=\int_{\bR^d}e^{-i(\vs-\vt)\cdot \vxi}\varphi(\vxi)d\vxi,\ \ \vs,\vt\in\bR^d,
\end{equation}
where $\vs\cdot \vt$ stands for the standard inner product on $\bR^d$, and $\varphi\in L^1(\bR^d)$, the space of Lebesgue integrable functions on $\bR^d$. One should not confuse $L^1(\bR^d)$ with $\ell^1(\bR^d)$. The latter one is defined with respect to the counting measure on $\bR^d$ while the first one is with respect to the Lebesgue measure. Note that $K$ is bounded and continuous on $\bR^d\times\bR^d$. We give a sufficient condition for so defined a function $K$ to satisfy \Hprimelabel.

\begin{prop}\label{fullsupport}
Let $K$ be given by (\ref{translationinvariantK}). If $\varphi$ is nonzero almost everywhere on $\bR^d$ with respect to the Lebesgue measure then $K$ satisfies \Hprimelabel.
\end{prop}
\begin{proof}
Suppose that there exists $\vc\in\ell^1(\bN)$ and pairwise distinct points $\vs_j\in\bR^d$, $j\in\bN$ such that
$$
\sum_{j=1}^\infty c_jK(\vs_j, \vt)=0\mbox{ for all }\vt\in\bR^d.
$$
This equation can be reformulated by (\ref{translationinvariantK}) as
$$
\int_{\bR^d}\biggl(\sum_{j=1}^\infty c_j e^{-i \vs_j\cdot \vxi}\biggr)\varphi(\vxi)e^{i \vt\cdot\vxi}d\vxi=0\mbox{ for all } \vt\in\bR^d.
$$
It follows that for almost every $\vxi\in\bR^d$ with respect to the Lebesgue measure
$$
\biggl(\sum_{j=1}^\infty c_j e^{-i \vs_j\cdot \vxi}\biggr)\varphi(\vxi)=0.
$$
By the assumption on $\varphi$,
$$
\sum_{j=1}^\infty c_j e^{-i \vs_j\cdot \vxi}=0\mbox{ for almost every }\vxi\in\bR^d.
$$
Note that the function on the left hand side above is continuous on $\vxi$. We hence obtain that the Fourier transform of the discrete measure
$$
\nu(A):=\sum_{\vs_j\in A}c_j\mbox{ for every Borel subset }A\subseteq\bR^d
$$
is zero. Consequently, $\nu$ is the zero measure, implying $\vc=0$.
\end{proof}

We next present a particular example as a corollary to Proposition \ref{fullsupport}.
\begin{cor}\label{cporiginal}
If $\phi$ is nontrivial continuous function on $\bR^d$ with a compact support then $K(\vs,\vt)=\phi(\vs-\vt)$, $\vs,\vt\in\bR^d$ satisfies \Hprimelabel.
\end{cor}
\begin{proof}
We regard $\phi$ as a tempered distribution and note by the Paley-Wiener theorem that the Fourier transform of $\phi$ is real-analytic on $\bR^d$. Therefore, the Fourier transform of $\phi$ is nonzero everywhere on $\bR^d$ except at a subset of zero Lebesgue measure. The arguments similar to those in the proof of the last proposition hence apply.
\end{proof}

We next present by Proposition \ref{fullsupport} and Corollary \ref{cporiginal} several examples of $K$ that satisfy \Hprimelabel and hence can be used to construct RKBS with the $\ell^1$ norm. Such functions include:
\begin{itemize}
\item[--] the exponential kernel
$$
K(\vs,\vt)=\exp(-\|\vs-\vt\|_{\ell^1(\bN_d)})=\frac1{\pi^d}\int_{\bR^d}e^{-i (\vs-\vt)\cdot\vxi}\prod_{j=1}^d\frac1{1+\xi_j^2}d\vxi,\ \ \vs,\vt\in\bR^d,
$$
where for $\vs\in\bR^d$, $\|\vs\|_2$ is its standard Euclidean norm on $\bR^d$.

\item[--] the Gaussian kernel
\begin{equation}\label{gaussian}
K(\vs,\vt)=\exp\biggl(-\frac{\|\vs-\vt\|_2^2}{\sigma}\biggr)=\biggl(\frac{\sqrt{\sigma}}{2\sqrt{\pi}}\biggr)^d\int_{\bR^d}e^{-i(\vs-\vt)\cdot\vxi}\exp(-\frac{\sigma}4\|\vxi\|_2^2)d\vxi,\ \ \vs,\vt\in\bR^d.
\end{equation}

\item[--] inverse multiquadrics
\begin{equation}\label{multiquadric}
K(\vs,\vt)=\biggl(\frac1{1+\|\vs-\vt\|_2^2}\biggr)^\beta,\ \ \vs,\vt\in\bR^d,\ \ \beta>0,
\end{equation}
whose Fourier transform is given by the modified Bessel function and is positive almost everywhere on $\bR^d$ (see \cite{Wendland}, pages 52, 76 and 95).

\item[--] B-spline kernels
$$
K(\vs,\vt)=\prod_{j=1}^d B_p(s_j-t_j),\ \ \vs,\vt\in\bR^d,
$$
where $s_j$ is the $j$-th component of $\vs$ and $B_p$ denotes the $p$-th order B-spline, $p\ge 2$. B-spline kernels satisfies \Hprimelabel as they are given by bounded continuous functions of compact support.

\item[--] radial basis functions of compact support, including Wu's functions \cite{Wu1995} and Wendland's functions \cite{Wendland}. Such functions are of the form $K(\vs,\vt)=\phi(\|\vs-\vt\|_2)$, $\vs,\vt\in\bR^d$, where $\phi$ is a compactly supported univariate function dependent on the dimension $d$. We give two examples for $d=3$:
    $$
    \phi(r):=(1-r)^2_+\mbox{ and }\phi(r):=(1-r)^4_+(1+4r),\ \ r\ge 0
    $$
    where $t_+:=\max\{0,t\}$ for $t\in\bR$. These functions satisfy \Hprimelabel by Corollary \ref{cporiginal}.
\end{itemize}

On the other hand, a translation invariant $K$ does not satisfy \Hprimelabel if its Fourier transform is compactly supported, as indicated in the next result.
\begin{prop}\label{compactsupport}
If $\varphi\in L^1(\bR^d)$ is compactly supported on $\bR^d$ then $K$ given by (\ref{translationinvariantK}) does not satisfy \Hprimelabel.
\end{prop}
\begin{proof}
Without lost of generality, we may assume that $\supp \varphi\subseteq [-1,1]^d$. Choose a nontrivial infinitely continuously differentiable function $\phi$ that is supported on $[-\pi,\pi]^d$ and vanishes on $[-1,1]^d$. We expand $\phi$ to a Fourier series
$$
\phi(\vxi)=\sum_{\vj\in\bZ^d} c_{\vj} \me^{-i \vj\cdot \vxi},\ \ \vxi\in[-\pi,\pi]^d,
$$
where $c_{\vj}$ is the Fourier coefficient of $\phi$. Note that $\{c_{\vj}: \vj\in\bZ^d\}\in\ell^1(\bZ^d)$ as $\phi$ is infinitely continuously differentiable on $[-\pi,\pi]^d$. By arguments in the proof of Proposition \ref{fullsupport},
$$
\sum_{\vj\in\bZ^d}c_{\vj} K(\vj,\vt)=\int_{\bR^d}\biggl(\sum_{\vj\in\bZ^d} c_{\vj} \me^{-i \vj\cdot \vxi}\biggr)\varphi(\vxi) \me^{i \vt\cdot \vxi}d\vxi,\ \ \vt\in\bR^d.
$$
By our construction,
$$
\biggl(\sum_{\vj\in\bZ^d} c_{\vj} \me^{-i \vj\cdot \vxi}\biggr)\varphi(\vxi)=0\mbox{ for all }\vxi\in\bR^d,
$$
which implies $\sum_{\vj\in\bZ^d}c_{\vj} K(\vj,\cdot)=0$. Moreover, $c_{\vj}\ne0$ for at least one $\vj\in\bZ^d$ because $\phi$ is nontrivial. We obtain that $K$ does not satisfy \Hprimelabel.
\end{proof}

By Proposition \ref{compactsupport}, the sinc kernel
$$
K(\vs,\vt):=\sinc(\vs-\vt):=\prod_{j=1}^d\frac{\sin(\pi(s_j-t_j))}{\pi(s_j-t_j)},\ \ \vs,\vt\in\bR^d
$$
does not satisfy \Hprimelabel. As a consequence, it can not yield an RKBS with the $\ell^1$ norm by the procedures introduced in this section. Similar arguments as those in the proof of Proposition \ref{compactsupport} are able to show that if $\nu$ is a compactly supported Borel measure on $\bR^d$ of finite total variation then the following function
$$
K(\vs,\vt):=\int_{\bR^d}e^{-i(\vs-\vt)\cdot \vxi} d\nu(\vxi),\ \ \vs,\vt\in\bR^d
$$
does not satisfy \Hprimelabel. Instances include the class of Bessel-based radial functions \cite{Fornberg2006} where the Borel measure is the dirac delta measure on the unit sphere of the Euclidean space.

\section{Representer Theorems in RKBS with the $\ell^1$ Norm} \label{repthemsec}
Up to now our arguments have relied on Admissibility Assumptions \nonsingassum --\Hprimelabel.   In this section the final assumption, \Addptsmalllabel, is invoked to guarantee that the representer theorem should hold for the constructed RKBS. A regularized learning scheme in the RKBS $\cB$ constructed by (\ref{cbl1norm}) can be generally expressed as finding $f_0$ such that
\begin{equation}\label{regularization1}
f_0 = \argmin_{f\in\cB}[V(f(\bx))+\mu \phi(\|f\|_\cB)],
\end{equation}
where $\bx:=\{{x}_j\in X:j\in\bN_n\}$, $n\in\bN$, is the sequence of given pairwise distinct sampling points, $f(\bx):=(f({x}_j):j\in\bN_n)\in\bC^n$, $V:\bC^n\to\bR_+$ is a loss function, $\mu$ is a positive regularization parameter, and $\phi:\bR_+\to\bR_+$ is a nondecreasing regularization function. Here, $\bR_+:=[0,+\infty)$. The loss function and regularization function should satisfy some minimal requirements for the learning scheme (\ref{regularization1}) to be useful. This consideration gives rise to the following definition.

\begin{defn} A regularized learning scheme (\ref{regularization1}) is said to be \emph{acceptable} if $V$ and $\phi$ are continuous and
\begin{equation}\label{tendtoinfty}
\lim_{t\to\infty}\phi(t)=+\infty.
\end{equation}
\end{defn}

It is possible that the solution to (\ref{regularization1}) is non-unique, and in that case we are only interested in finding one possible solution.

We now introduce the main concept of this section.

\begin{defn}
 The space $\cB$ is said to \emph{satisfy the linear representer theorem for regularized learning} if every acceptable regularized learning scheme (\ref{regularization1}) has a minimizer of the form
\begin{equation}\label{representer1}
f_0=\sum_{j=1}^n c_j K({x}_j,\cdot),
\end{equation}
where $c_j$'s are constants. In other words, there exists a solution $f_0$ lying in the finite dimensional subspace $\cS^{\bx}:=\linearspan\{K({x}_j,\cdot):j\in\bN_n\}$.
\end{defn}

An RKHS with $K$ being its reproducing kernel in the usual sense always satisfies the linear representer theorem \cite{Kimeldorf1971}. The result for uniformly convex and uniformly Fr\'{e}chet differentiable pre-RKBS with a reproducing kernel given by the semi-inner product was established in \cite{Zhang2009,Zhangjogo}. For more information on this important property for RKHS and vector-valued RKHS, see, for example, \cite{Argyriou2009,Micchelli2005a,Scholkopf2001} and the references cited therein.

Our purpose is to discuss the conditions on $K$ such that $\cB$ satisfies the linear representer theorem. The representer theorem for (\ref{regularization1}) is closely related to the representer theorem for the minimal norm interpolation problem. In the RKHS case, an equivalence was proved in \cite{Micchelli1994}. We shall follow the approach to consider the minimal norm interpolation in $\cB$ first. For any $\vy\in \bC^n$, set $\cI_{\bx}(\vy)$ to be the subset of functions in $\cB$ that interpolate the specified data, namely,
$ \cI_{\bx}(\vy):=\{f\in\cB :f(\bx)=\vy\}$.
A minimal norm interpolant in $\cB$ is a function $f_{\min}$ satisfying
\begin{equation}\label{mni}
f_{\min}= \argmin \{\|f\|_\cB : f\in\cI_{\bx}(\vy)\}.
\end{equation}
Again, in the case of a non-unique solution, we are only interested in obtaining one solution.  Since $K[\bx]$ is nonsingular, one sees that the typically infinite dimensional  $\cI_{\bx}(\vy)$ always has a non-empty intersection with $\cS^{\bx}$,
for all $\vy\in \bC^n$ and pairwise distinct $\bx\subseteq X$.

\begin{defn}An RKBS $\cB$ is said to \emph{satisfy the linear representer theorem for minimal norm interpolation} if for any choice of data, $\bx$ and $\vy$, there is a minimal norm interpolant, (\ref{mni}), lying in $\cS^{\bx}$.
\end{defn}

We shall show that $\cB$ satisfies the linear representer theorem for regularized learning if and only if it does so for minimal norm interpolation. We first prove one direction of the equivalence.

\begin{lem}\label{mnileadtoreg}
If $\cB$ satisfies the linear representer theorem for the minimal norm interpolation, then it also does so for regularized learning.
\end{lem}
\begin{proof} Let $V$, $\phi$, and $\mu$ be arbitrary, but fixed according to the conditions that \eqref{regularization1} be an acceptable regularization scheme. For an arbitrary function $f$ in $\cB$. We let $f_0$ be the minimizer of $\inf_{g\in\cI_{\bx}(f(\bx))}\|g\|_\cB$ that has the form (\ref{representer1}). Then $f_0(\bx)=f(\bx)$ and $\|f_0\|_\cB\le \|f\|_\cB$. As a consequence, $V(f_0(\bx))=V(f(\bx))$ but $\phi(\|f_0\|_\cB)\le\phi(\|f\|_\cB)$ as $\phi$ is nondecreasing. It follows that
$$
\inf_{f\in\cB}V(f(\bx))+\mu\phi(\|f\|_\cB)=\inf_{f\in\cS^{\bx}}V(f(\bx))+\mu\phi(\|f\|_\cB).
$$
By (\ref{tendtoinfty}), there exists a positive constant $\alpha$ such that
$$
\inf_{f\in\cS_{\bx}}V(f(\bx))+\mu\phi(\|f\|_\cB)=\inf_{f\in\cS^{\bx},\|f\|_\cB\le\alpha}V(f(\bx))+\mu\phi(\|f\|_\cB).
$$
Note that the functional we are minimizing is continuous on $\cB$ by the assumption on $V$, $\phi$ and by the continuity of point evaluation functionals on $\cB$. By the elementary fact that a continuous function on a compact metric space attains its minimum in the space, (\ref{regularization1}) has a minimizer that belongs to $\{f\in\cS^{\bx}:\|f\|_\cB\le\alpha\}$. Therefore, $\cB$ satisfies the linear representer theorem.
\end{proof}

For the other direction, it suffices to consider a class of regularization functionals with a particular choice of $V$ and $\phi$.  In the limit of vanishing $\mu$ we recover the minimal norm interpolant.

\begin{lem}\label{reginterp}
If $\cB$ satisfies the linear representer theorem for regularized learning, then it also satisfies the linear representer theorem for minimal norm interpolation.
\end{lem}
\begin{proof}
We shall follow the idea in \cite{Micchelli1994}. Choose any $n \in \bN_n$, any  $\bx =\{{x}_j\in X:j\in\bN_n\}$ with pairwise distinct elements, and any $\vy\in\bC^n$. For every $\mu>0$, let $f_{0,\mu} \in \cS^{\bx}$ be a minimizer of \eqref{regularization1} with the choice of
\begin{equation}\label{costBanach}
V(f(\bx))= \|f(\bx)-\vy\|^2_2, \qquad \phi(t) = t.
\end{equation}
Here, $\|\cdot\|_2$ is the standard Euclidean norm on $\bC^n$.
Defining the $1\times n$ row vector function by
$$
K^{\bx}({x}):=(K({x}_j, {x}):j\in\bN_n)\mbox{ for all }{x}\in X.
$$
It follows that $f_{0,\mu} = K^{\bx}(\cdot) \vc_\mu$ for some $\vc_\mu \in\bC^n$.  Then we have
$$
\|K[\bx] \vc_\mu-\vy\|^2_2=\|f_{0,\mu}(\bx)-\vy\|^2_2 \le V(f_{0,\mu}) + \mu \phi(\|f_{0,\mu}\|_{\cB}) \le V(0) + \mu \phi(\|0\|_{\cB})=\|\vy\|_2^2.
$$
As $K[\bx]$ is nonsingular, the above inequality implies that $\{\vc_\mu : \mu>0\}$ forms a bounded set in $\bC^n$. By restricting to a subsequence if necessary, we may hence assume that $\vc_\mu$ converges to some $\vc_0\in \bC^n$ as $\mu$ goes to zero. We shall show that $f_{0,0}:=K^{\bx}(\cdot) \vc_0 \in \cS^{\bx}$ is a minimal norm interpolant.

Since $\vc_\mu$ converges to $\vc_0$ as $\mu$ tends to zero, we first get
\begin{equation}\label{regint1}
\lim_{\mu\to 0}\|f_{0,\mu}-f_{0,0}\|_\cB=\lim_{\mu\to 0}\| \vc_{\mu} - \vc_0\|_{\ell^1(\bN_n)}=0.
\end{equation}
Since point evaluation functionals are continuous on $\cB$, we obtain by (\ref{regint1})
\begin{equation}\label{reginterpeq2}
f_{0,0}({x}_j)= \lim_{\mu\to 0} f_{0,\mu}({x}_j) \mbox{ for all }j\in\bN_n.
\end{equation}
Now let $g$ be an arbitrary interpolant, i.e., an arbitrary element of $\cI_{\bx}(\vy)$. As $f_{0,\mu}$ is a minimizer of \eqref{regularization1} with the choice (\ref{costBanach}), it follows that
\begin{equation}\label{reginterpeq3}
\|f_{0,\mu}(\bx) - \vy\|^2_2 + \mu\|f_{0,\mu} \|_{\spB} \leq \|g(\bx) - \vy\|^2_2 + \mu\|g \|_{\spB} = \mu \|g \|_{\spB}.
\end{equation}
Letting $\mu\to0$ on both sides of the above inequality, we obtain by (\ref{reginterpeq2}) $\|f_{0,0}(\bx)-\vy\|^2_2=0$, which implies that $f_{0,0}$ is also an interpolant, i.e,. $f_{0,0}\in \spI_{\bx}(\vy)$. It also follows from (\ref{reginterpeq3}) that $\|f_{0,\mu} \|_{\spB} \leq  \|g \|_{\spB}$ for all $\mu>0$, which together with (\ref{regint1}) implies $\|f_{0,0} \|_{\spB} \leq  \|g \|_{\spB}$. Since $g$ is an arbitrary function in $\cI_{\bx}(\vy)$ and $f_{0,0}\in\cI_{\bx}(\vy)$, we see that $f_{0,0}$ is a minimal norm interpolant, i.e., a solution of (\ref{mni}). The proof is complete.
\end{proof}

Combining Lemmas \ref{mnileadtoreg} and \ref{reginterp}, we reach the characterization for $\cB$ to satisfy the linear representer theorem.
\begin{prop}\label{equivalenttomni}
The space $\cB$ satisfies the linear representer theorem for regularized learning if and only if $\cB$ satisfies the linear representer theorem for minimal norm interpolation.
\end{prop}

In view of the above result, we shall focus on necessary and sufficient conditions for the minimal norm interpolation in $\cB$ to satisfy the linear representer theorem. To this end, we begin with the simplest case when only one more sampling point is added to $\bx$. Recall the definition of $K_\bx(x)$ from the introduction. It is worthwhile to point out that $K_\bx(x)$ is in general not the transpose of $K^\bx(x)$ as $K$ is not required to be symmetric.

\begin{lem}\label{onemorepoint}
Let $\bx =\{{x}_j\in X:j\in\bN_n\}$ have pairwise distinct elements, let ${x}_{n+1}$ be an arbitrary point in $X \backslash \bx$, and set $\overline{\bx}:=\{{x}_j:j\in\bN_{n+1}\}$. It follows that the minimum norm interpolant in $\spS^{\overline{\bx}}$ is the same as the minimum norm interpolant in $\spS^{\bx}$, i.e.,
\begin{equation}\label{onemorepointeq}
\min_{f\in \spI_{\bx}(\vy) \cap \spS^{\overline{\bx}}} \| f\|_{\spB} =\min_{f\in \spI_{\bx}(\vy) \cap \spS^{\bx}} \| f\|_{\spB}\mbox{ for all }\vy\in \bC^n,
\end{equation}
if and only if \eqref{H2condition} holds true.
\end{lem}
\begin{proof}
Notice that $\cI_{\bx}(\vy)\cap\cS^{\bx}$ has only one function $f=K^{\bx}(\cdot)K[\bx]^{-1}\vy$. We next estimate the norm of functions in $\spI_{\bx}(\vy) \cap \spS^{\overline{\bx}}$. Let $g\in \cI_{\bx}(\vy) \cap \cS^{\overline{\bx}}$ and $b:=g({x}_{n+1})$. Note that $g$ is uniquely determined by $b$ as it has already satisfied the interpolation condition $g(\bx)=\vy$. In fact, as $K[\overline{\bx}]$ is nonsingular,
$g=K^{\overline{\bx}}(\cdot)K[\overline{\bx}]^{-1}\overline{\vy}$,
where $\overline{\vy}=(\vy^T,b)^T\in\bC^{n+1}$. Direct computations show that
$$
K[\overline{\bx}]^{-1}  \overline{\vy} = \left( \begin{matrix} K[\bx] & K_{{\bx}}({x}_{n+1}) \\ K^{\bx}({x}_{n+1}) & K({x}_{n+1}, {x}_{n+1})  \end{matrix} \right)^{-1} \begin{pmatrix} \vy \\ b \end{pmatrix}= \begin{pmatrix} K[\bx]^{-1}\vy + \frac{q}{p} K[\bx]^{-1} K_{{\bx}}({x}_{n+1})  \\[0.3em]  - \frac{q}{p}  \end{pmatrix},
$$
where $p :=K({x}_{n+1}, {x}_{n+1}) -  K^{\bx}({x}_{n+1})K[\bx]^{-1}K_{{\bx}}({x}_{n+1})$ and $q:= K^{\bx}({x}_{n+1}) K[\bx]^{-1}\vy - b$.

We now show sufficiency. If (\ref{H2condition}) holds true then we have
$$
\begin{array}{rl}
\|g\|_\cB&=\|K[\overline{\bx}]^{-1}  \overline{\vy}\|_{\ell^1(\bN_{n+1})} \geq \|K[\bx]^{-1}\vy \|_{\ell^1(\bN_n)} - \left\| (K[\bx])^{-1} K_{{\bx}}({x}_{n+1})\right\|_{\ell^1(\bN_n)}|\frac qp|+ |\frac qp|\\
&\geq \|K[\bx]^{-1}\vy \|_{\ell^1(\bN_n)}= \|f\|_\cB,
 \end{array}
$$
which implies
\begin{equation*}
\min_{f\in \cI_{\bx}(\vy) \cap \cS^{\overline{\bx}}} \| f\|_{\spB} \geq \min_{f\in \cI_{\bx}(\vy) \cap \spS^{\bx}} \| f\|_{\spB}.
\end{equation*}
Since $\cS^{\bx}\subseteq \cS^{\overline{\bx}}$,
$$
\min_{f\in \spI_{\bx}(\vy) \cap \spS^{\overline{\bx}}} \| f\|_{\spB} \le \min_{f\in \spI_{\bx}(\vy)\cap\cS^{\bx}}\| f\|_{\spB}.
$$
Thus, (\ref{onemorepointeq}) holds true.

On the other hand, if (\ref{onemorepointeq}) is always true for all $\vy\in\bC^n$ then we must have
$$
\|K[\overline{\bx}]^{-1}\overline{\vy}\|_{\ell^1(\bN_{n+1})}\ge \|K[\bx]^{-1}\vy \|_{\ell^1(\bN_n)}\mbox{ for all }\vy\in\bC^n\mbox{ and }b\in\bC.
$$
In particular, the choices $\vy = K_{{\bx}}({x}_{n+1})$ and $b=K^{\bx}({x}_{n+1}) K[\bx]^{-1}K^T_{{x}}({x}_{n+1})+p$ yields that
\begin{equation*}
\|K[\overline{\bx}]^{-1}\overline{\vy}\|_{\ell^1(\bN_{n+1})}=\left\|\begin{pmatrix} {\boldsymbol{0}} \\ 1 \end{pmatrix}\right\|_{\ell^1(\bN_{n+1})} =1 \quad \mbox{and} \quad \|K[\bx]^{-1}\vy \|_{\ell^1(\bN_n)}= \left\| (K[\bx])^{-1} K_{{\bx}}({x}_{n+1})\right\|_{\ell^1(\bN_n)}.
\end{equation*}
Combing the above two equations proves (\ref{H2condition}). The proof is complete.
\end{proof}

We are now ready to present one of the main results in this paper.

\begin{thm}\label{repinterp}
Every minimal norm interpolant (\ref{mni}) in $\cB$ satisfies the linear representer theorem if and only if (\ref{H2condition}) holds true for all $n\in\bN$ and all pairwise distinct sampling points ${x}_j\in X$, $j\in\bN_{n+1}$.
\end{thm}
\begin{proof}
The minimal norm interpolant (\ref{mni}) satisfies the linear representer theorem if and only if
$$
\min_{g\in\cI_{\bx}(\vy)}\|g\|_\cB=\min_{f\in\spI_{\bx}(\vy) \cap \spS^{\bx}}\|f\|_\cB.
$$
Therefore, if the above equation holds true then since $\spI_{\bx}(\vy) \cap \spS^{\bx}\subseteq \spI_{\bx}(\vy) \cap \spS^{\overline{\bx}}\subseteq \cI_{\bx}(\vy)$, we obtain (\ref{onemorepointeq}). By Lemma \ref{onemorepoint}, (\ref{H2condition}) is true for every ${x}_{n+1}\in X$.

It remains to prove the sufficiency. We shall first show $\|g\|_\cB\ge \min_{f\in\spI_{\bx}(\vy) \cap \spS^{\bx}}\|f\|_\cB$ for all $g\in \cI_{\bx}(\vy)\cap \cB_0$. To this end, we express $g$ as $g=\sum_{j=1}^m c_j K({x}_j,\cdot)$ for some $m\geq n$ and pairwise distinct $\{{x}_j: j\in\bN_m\} \subseteq {X}$. This can always be done by adding some sampling points, setting the corresponding coefficients to be zero, and relabeling if necessary. We let $y_j:=g({x}_j)$, $j\in\bN_m$, $\vu_l:=(y_j: j\in\bN_l)$, and $\bv_l=\{{x}_j:j\in\bN_l\}$ for $1\leq l\leq m$. Note that $\vy=\vu_n$ and $\bx=\bv_n$. It follows that $g\in \spI_{\bv_m}(\vu_m) \cap \spS^{\bv_m}$ and thus,
\begin{equation*}
\|g \|_{\spB}\geq \min_{f\in \spI_{\bv_m}(\vu_m) \cap \spS^{\bv_m}} \| f\|_{\spB}.
\end{equation*}
Since $\spI_{\bv_m}(\vu_m) \subseteq \spI_{\bv_{m-1}}(\vu_{m-1})$, we apply Lemma \ref{onemorepoint} to get
\begin{equation*}
\min_{f\in \spI_{\bv_m}(\vu_m) \cap \spS^{\bv_m}} \| f\|_{\spB}\geq  \min_{f\in \spI_{\bv_{m-1}}(\vu_{m-1}) \cap \spS^{\bv_m}} \| f\|_{\spB} = \min_{f\in \spI_{\bv_{m-1}}(\vu_{m-1}) \cap \spS^{\bv_{m-1}}} \| f\|_{\spB}.
\end{equation*}
It follows that
$$
\|g \|_{\spB}\geq \min_{f\in \spI_{\bv_{m-1}}(\vu_{m-1}) \cap \spS^{\bv_{m-1}}} \| f\|_{\spB}.
$$
Repeating this process, we reach
\begin{equation}\label{repinterp1}
\|g \|_{\spB}\geq \min_{f\in \spI_{\bv_n}(\vu_n) \cap \spS^{\bv_n}} \| f\|_{\spB}=\min_{f\in\spI_{\bx}(\vy) \cap \spS^{\bx}}\|f\|_\cB\mbox{  for all }g\in\cI_{\bx}(\vy)\cap \cB_0.
\end{equation}

Now let $g\in\cI_{\bx}(\vy)$ be arbitrary but fixed. Then there exists a sequence of functions $\{ g_j\in\cB_0:j\in\bN \}$ that converges to $g$ in $\cB$. We let $f$ and $f_j$ be the function in $\spS^{\bx}$ such that $f(\bx)=\vy$ and $f_j(\bx)=g_j(\bx)$, $j\in\bN$. They are explicitly given by
\begin{equation*}
f= K^{\bx}(\cdot)K[\bx]^{-1}g(\bx) \quad \mbox{and} \quad f_j= K^{\bx}(\cdot)K[\bx]^{-1}g_j(\bx), \ \ j\in\bN.
\end{equation*}
Since $g_j$ converges to $g$ in $\cB$ and point evaluation functionals are continuous on $\cB$, $g_j(\bx)\to g(\bx)$ as $j\to\infty$. As a result, $\limd_{j\to \infty} \|f -f_j\|_{\spB}=0$. By \eqref{repinterp1}, $\| g_j\|_{\spB} \geq \| f_j\|_{\spB}$ for all $j\in\bN$. We hence obtain that $\|g\|_\cB\ge \|f\|_\cB$. Therefore,
$$
\min_{g\in\cI_{\bx}(\vy)}\|g\|_\cB\ge \min_{f\in\spI_{\bx}(\vy) \cap \spS^{\bx}}\|f\|_\cB.
$$
The reverse direction of the inequality is clear as $\spI_{\bx}(\vy) \cap \spS^{\bx}\subseteq \cI_{\bx}(\vy)$.
\end{proof}

We draw the following conclusion by Theorems \ref{equivalenttomni} and \ref{repinterp}.
\begin{cor}\label{repincb}
Every acceptable regularized learning scheme of the form (\ref{regularization1}) has a minimizer of the form (\ref{representer1}) if and only if the function $K$ satisfies the property (\ref{H2condition}).
\end{cor}

In the last part of the section, we briefly discuss the linear representer theorem in $\cB^\sharp$ under the same assumption that $K$ is bounded and satisfies \Hprimelabel. By Theorem \ref{H1condition}, $\cB^\sharp$ is an RKBS on $X$. Likewise, we call a regularized learning scheme
\begin{equation}\label{regularization2}
f_0 = \argmin_{f\in\cB^\sharp}V(f(\bx))+\mu \phi(\|f\|_{\cB^\sharp})
\end{equation}
{\it acceptable} if $V$ and $\phi$ are continuous and (\ref{tendtoinfty}) is satisfied by $\phi$. The space $\cB^\sharp$ is said to satisfy the linear representer theorem if every acceptable learning scheme (\ref{regularization2}) has a minimizer of the following form
\begin{equation}\label{representer2}
f_0=\sum_{j=1}^n c_j K(\cdot,{x}_j),
\end{equation}
where $c_j$'s are constants. We follow similar approaches to those used for $\cB$ to study this important property on $\cB^\sharp$.

\begin{prop}\label{reginterpBsharp}
Let $\bx\subseteq X$ have pairwise distinct elements. Every acceptable regularized learning scheme (\ref{regularization2}) in $\cB^\sharp$ has a minimizer, $f_0$ lying in $\cS_{\bx}:=\linearspan\{K(\cdot, {x}_j):j\in\bN_n\}$ if and only if there is a minimal norm interpolant,
\begin{equation}\label{mni2}
f_{\min}:=\argmin_{f\in\cB^\sharp,f(\bx)=\vy}\|f\|_{\cB^\sharp}
\end{equation}
lying in $\cS_{\bx}$ for all $\vy\in\bC^n$.
\end{prop}
\begin{proof}
The arguments of the proof are similar to those for $\cB$. One only needs to note that although the norm of a function in $\cB^\sharp$ may not be known, any two norms on the finite dimensional vector space $\cS_\bx$ are equivalent.
\end{proof}

To study conditions ensuring that the minimal norm interpolation (\ref{mni2}) satisfies the linear representer theorem, we first identify a specific form of the norm $\|\cdot\|_{\cB^\sharp_0}$ under the assumption that $K$ satisfies (\ref{H2condition}). Notice that a function $f_{\vc}=\sum_{j=1}^nc_jK(\cdot,x_j) \in \cS_{\bx} \subseteq \cB_0^\sharp$ can be represented as $f_{\vc}=\vc^T K_\bx(\cdot)$.

\begin{lem}\label{anotherBsharpnorm}
Let $\bx$ have pairwise distinct elements. The function $K$ satisfies (\ref{H2condition}) if and only if
\begin{equation}\label{anotherBsharpnorm1}
\| f_{\vc} \|_{\spB^\sharp} = \| \vc^TK[\bx]\|_\infty\mbox{ for all } f_{\vc}=\vc^TK_\bx(\cdot),\ \ \vc\in\bC^n,
\end{equation}
where $\|\cdot\|_\infty$ denotes the maximum norm on $\bC^n$.
\end{lem}
\begin{proof}
Suppose that $K$ satisfies (\ref{H2condition}) for all $x_{n+1}\in X\setminus \bx$. Then we have for all $x\in X$ that $\|K[\bx]^{-1}K_\bx(x)\|_{\ell^1(\bN_n)}\le 1$. Let $\vc\in\bC^n$ and $x\in X$. It follows from this inequality that
\begin{equation*}
|\vc^TK_\bx(x)| = |\vc^TK[\bx]K[\bx]^{-1}K_\bx(x)|\le \|\vc^TK[\bx]\|_\infty \|K[\bx]^{-1}K_\bx(x)\|_{\ell^1(\bN_n)}\le \|\vc^TK[\bx]\|_\infty,
\end{equation*}
which implies by Lemma \ref{BsharpNorm} that for $f_c=\vc^TK_\bx(\cdot)$
\begin{equation*}
\| f_{\vc} \|_{\spB^\sharp} = \|\vc^TK_\bx(\cdot) \|_{L^\infty(X)} \leq  \| \vc^TK[\bx] \|_\infty.
\end{equation*}
The other direction of the inequality is clear as we have
\begin{equation*}
 \|\vc^T K[\bx] \|_\infty = \max\{|\vc^TK_\bx(x_j)|: j\in\bN_n\} \leq \|\vc^TK_\bx(\cdot) \|_{L^\infty(X)} =\| f_{\vc} \|_{\spB^\sharp}.
\end{equation*}

It remains to show that (\ref{anotherBsharpnorm1}) implies (\ref{H2condition}). We prove this by construction. For any $x_{n+1}\in {X}$, we can find a nonzero vector $\vc \in \bC^n$ such that
\begin{equation*}
|\vc^TK_\bx(x_{n+1})| = |\vc^TK[\bx]K[\bx]^{-1}K_\bx(x_{n+1})| = \|\vc^TK[\bx] \|_\infty \|K[\bx]^{-1}K_\bx(x_{n+1})\|_{\ell^1(\bN_n)}.
\end{equation*}
We then let $f_{\vc}=\vc^TK_\bx(\cdot)$ and obtain by (\ref{anotherBsharpnorm1})
\begin{equation*}
 \|\vc^TK[\bx] \|_\infty \|K[\bx]^{-1}K_\bx(x_{n+1})\|_{\ell^1(\bN_n)}=|f_{\vc}(x_{n+1})|\le \|f_{\vc}\|_{L^\infty(X)}=\|f_{\vc}\|_{\cB^\sharp}=\|\vc^T K[\bx] \|_\infty,
\end{equation*}
which implies (\ref{H2condition}) for $\vc^TK[\bx]$ is not the zero vector. The proof is complete.
\end{proof}

We now show that (\ref{H2condition}) is sufficient for $\cB^\sharp$ to satisfy the linear representer theorem.

\begin{thm}\label{repinterpBsharp}
If $K$ satisfies (\ref{H2condition}) then $\cB^\sharp$ satisfies the linear representer theorem.
\end{thm}
\begin{proof}
Suppose that (\ref{H2condition}) holds true. By Lemma \ref{reginterpBsharp}, it suffices to show that the minimal norm interpolation (\ref{mni2}) has a minimizer of the form (\ref{representer1}). We shall prove this by directly showing that $f_0=\vy^TK[\bx]^{-1}K_\bx(\cdot)$ is a minimizer for (\ref{mni2}). Let $f$ be an arbitrary function in $\cB^\sharp$ such that $f(\bx)=\vy$. Then we have by Lemma \ref{BsharpNorm}
\begin{equation*}
\| f\|_{\spB^\sharp}  = \| f\|_{L^\infty(X)} \geq \| f(\bx)\|_\infty=\|\vy\|_\infty.
\end{equation*}
By Lemma \ref{anotherBsharpnorm},
$$
\| f_0\|_{\spB^\sharp} = \|\vy^TK[\bx]^{-1}K[\bx]\|_\infty=\|\vy\|_\infty.
$$
Combining the above two inequalities leads to $\|f_0\|_{\cB^\sharp}\le \|f\|_{\cB^\sharp}$. Therefore, \eqref{mni2} has the minimizer $f_0=\vy^TK[\bx]^{-1}K_\bx(\cdot)$ which has the form (\ref{representer2}).
\end{proof}

In the particular case when $X$ has a finite cardinality, we shall show that condition (\ref{H2condition}) is also necessary for $\cB^\sharp$ to satisfy the linear representer theorem.

\begin{prop}\label{repinterpBsharpNes}
If $X$ consists of finitely many points and $\cB^\sharp$ satisfies the linear representer theorem then \eqref{H2condition} holds true.
\end{prop}
\begin{proof}
Let $\vc\in\bC^n$ and $f_{\vc}=\vc^TK_\bx(\cdot)$. Under the assumptions, we get by Proposition \ref{reginterpBsharp} that $f_{\vc}$ is a minimizer for the minimal norm interpolation (\ref{mni2}) with $\vy=f_{\vc}(\bx)=(K[\bx])^T\vc$. Since $X$ has a finite cardinality and $K[\bx]$ is nonsingular for all pairwise distinct $\bx\subseteq X$, we can find a function $g\in\cB_0$ such that $g(\bx)=\vy$ and $\|g\|_{L^\infty(X)}\le \|\vy\|_\infty$. Since $f_{\vc}$ is a minimizer of (\ref{mni2}) and $g$ satisfies $g(\bx)=\vy$,
$$
\| f_{\vc} \|_{\spB^\sharp} \leq  \| g \|_{\spB^\sharp}=\|g\|_{L^\infty(X)}=\| \vy\|_\infty=\|(K[\bx]^T)\vc\|_\infty.
$$
On the other hand, we have by Lemma \ref{BsharpNorm}
\begin{equation*}
\| f_{\vc} \|_{\spB^\sharp} = \| f_{\vc} \|_{L^\infty(X)} \geq \|f_{\vc}(\bx)\|_\infty=\|(K[\bx]^T)\vc\|_\infty.
\end{equation*}
By the above two equations, (\ref{anotherBsharpnorm1}) holds true. By Lemma \ref{anotherBsharpnorm}, $K$ satisfies (\ref{H2condition}).
\end{proof}

One observes that the key ingredient in the proof of Proposition \ref{repinterpBsharpNes} is to extend a function on the discrete set $\bx$ to a function in $\cB^\sharp$ in a way that the supremum norm is preserved. In many cases, this is achievable without $X$ being a finite set. For instance, by the Tietze extension theorem in topology, such an extension exists when $X$ is a compact metric space and $K$ is a universal kernel \cite{MXZ2006} on $X$. Thus, for those input spaces $X$ and functions $K$, $\cB^\sharp$ satisfies the linear representer theorem if and only if (\ref{H2condition}) holds true.

\section{Examples of Admissible Kernels} \label{examplesec}

Recall the definition of admissible kernels from the introduction. Note that the first requirement (A1) in the definition implies (\ref{linearlyindependentK}). Theorem \ref{construction} is proved by combining Theorem \ref{H1condition} and Corollary \ref{repincb}. By this result, admissible kernels are crucial for our construction. Functions $K$ satisfying requirements (A1)--(A3) are usually relatively easy to find. Some examples have been presented before Proposition \ref{compactsupport} in Section 3. However, requirement (A4) could be somewhat demanding and rule out many commonly used kernels. We are able to present two examples of admissible kernels below.

The first example is Brownian bridge kernel that arises in the study of Brownian bridge stochastic process in statistics \cite{BerlinetThomas-Agnan}.

\begin{prop}\label{brownian}
The Brownian bridge kernel defined by
\begin{equation*}
K(s,t):= \min\{s,t \} -st, \quad s,t\in (0,1)
\end{equation*}
is an admissible kernel on the input space $X=(0,1)$.
\end{prop}
\begin{proof}
We start with validating requirement (A4). Let $0< x_1 < x_2 < \cdots < x_n< 1$ be given and $x\in(0,1)$ be different from $x_j$, $j\in\bN_n$. Direct computations show that
\begin{enumerate}
\item If $x< x_1$ then $K[\bx]^{-1}K_{{\bx}}(x) = \left(\frac{x}{x_1}, 0,\ldots,0 \right)^T$.

\item If $x>x_n$ then $K[\bx]^{-1}K_{{\bx}}(x) =\left(0,\ldots,0, \frac{1-x}{1- x_n} \right)^T$.

\item If $x_j< x < x_{j+1}$ for some $j\in\bN_{n-1}$ then
\begin{equation*}
K[\bx]^{-1}K_{{\bx}}(x)= \left(0, \ldots,0, \frac{x_{j+1} -x}{x_{j+1} -x_j}, \frac{x -x_j}{x_{j+1} -x_j}, 0, \ldots, 0 \right)^T.
\end{equation*}
\end{enumerate}
In all cases, it is straightforward to see $\left\|K[\bx]^{-1}K_{{\bx}}(x)\right\|_{\ell^1(\bN_n)} \leq 1$. Therefore, requirement (A4) is indeed fulfilled.

To verify the other three requirements, we first observe
$$
K(s,t)=\int_0^1 \Gamma_s(z)\Gamma_t(z)dz,\ \ s,t\in(0,1),
$$
where $\Gamma_x:=\chi_{(0,x)}-x$ with $\chi_{A}$ standing for the characteristic function of $A\subseteq (0,1)$. Suppose that $K[\bx]\vc=0$ for some $\vc\in\bC^n$. Then we have
$$
\int_0^1 \biggl|\sum_{j=1}^n c_j\Gamma_{x_j}(z)\biggr|^2dz=\vc^*K[\bx]\vc=0,
$$
which implies that
$$
\sum_{j=1}^n c_j \Gamma_{x_j}(z)=0\mbox{ for almost every }z\in[0,1].
$$
Clearly, $\Gamma_{x_j}$, $j\in\bN_n$ are linearly independent. Therefore, $c_j=0$ for all $j\in\bN_n$. Requirement (A1) is hence satisfied.

The function $K$ is clearly bounded by $1$. Suppose that for some $\vc\in\ell^1(\bN)$ and pairwise distinct $x_j\in(0,1)$, $j\in\bN$
$$
\sum_{j=1}^\infty c_j K(x_j,x)=\int_0^1\biggl(\sum_{j=1}^\infty c_j\Gamma_{x_j}(z)\biggr)\Gamma_x(z)dz=0\mbox{ for all }x\in (0,1).
$$
It implies that the function $\phi:=\sum_{j=1}^\infty c_j\Gamma_{x_j}$ is orthogonal to $\Gamma_x$ for all $x\in(0,1)$, that is,
$$
\int_0^x\phi(t)dt-x\int_0^1\phi(t)dt=0\mbox{ for all }x\in(0,1).
$$
Taking the derivative on both sides of the above equations yields that $\phi$ equals a constant $C$ almost everywhere on $[0,1]$. Namely,
$$
\sum_{j=1}^\infty c_j\chi_{[0,x_j]}-\sum_{j=1}^\infty c_jx_j=C\mbox{ almost everywhere}.
$$
We now take the derivative of both sides of the equation above in the distributional sense to get $\sum_{j\in\bN}c_j\delta_{x_j}=0$. Let $j$ be an arbitrary but fixed positive integer. We can find a sequence of infinitely continuously differentiable functions $\phi_k$, $k\in\bN$ such that $\|\phi_k\|_{L^\infty([0,1])}\le 1$, $\phi_k(x_j)=1$, and the Lebesgue measure of the set where $\phi_k$ is nonzero is less than or equal to $\frac1k$. For each $N\in\bN$, we have for sufficiently large $k$ that
$$
\phi_k(t_l)=0\mbox{ for all }l\in\bN_N\setminus\{j\}.
$$
We get for this $\phi_k$
$$
0=\biggl|\biggl(\sum_{l=1}^\infty c_l\delta_{x_l}\biggr)(\phi_k)\biggr|\ge |c_j|-\sum_{l>N}|c_l|.
$$
Since $\sum_{l>N}|c_l|$ converges to zero as $N\to\infty$, we have $c_j=0$. Therefore, $\vc=0$ for $j$ is arbitrary chosen.

We conclude that all the four requirements of an admissible kernel are fulfilled by the Brownian bridge kernel.
\end{proof}

The second example is the exponential kernel (also called the $C^0$ Mat\'{e}rn kernel).
\begin{prop}
The exponential kernel
\begin{equation}\label{exponential}
K(s,t):=\me^{-|s-t|}, \quad s,t\in \R
\end{equation}
is an admissible kernel on $\bR$.
\end{prop}
\begin{proof}
We have seen in Section 3 that this kernel satisfies requirements (A1)--(A3). It remains to check requirement (A4). Let $x_1 < x_2 < \cdots < x_n$ be given and $x\in\bR$ be different from $x_j$, $j\in\bN_n$. Direct computations show that
\begin{enumerate}
\item If $x< x_1$ then $K[\bx]^{-1}K_{{\bx}}(x) = \left(\me^{x-x_1}, 0,\ldots,0 \right)^T$.

\item If $x>x_n$ then $K[\bx]^{-1}K_{{\bx}}(x) =\left(0,\ldots,0, \me^{x_n-x} \right)^T$.

\item If $x_j< x < x_{j+1}$ for some $j\in\bN_{n-1}$ then
\begin{equation*}
K[\bx]^{-1}K_{{\bx}}(x)= \left(0, \ldots,0, \frac{\me^{x_{j+1} -x} -\me^{x-x_{j+1}}}{\me^{x_{j+1} -x_j} - \me^{x_{j} -x_{j+1}}}, \frac{\me^{x -x_j} -\me^{x_j-x}}{\me^{x_{j+1} -x_j} - \me^{x_{j} -x_{j+1}}}, 0, \ldots, 0 \right)^T.
\end{equation*}
\end{enumerate}
In all cases, $\left\|K[\bx]^{-1}K_{{\bx}}(x)\right\|_{\ell^1(\bN_n)} \leq 1$. The proof is complete.
\end{proof}

%

Finally, we remark that by numerical experiments, the Gaussian kernel
$$
K(s,t)=\exp\biggl(-\frac{(s-t)^2}{\sigma}\biggr), \ \ s,t\in \bR
$$
does not satisfy (A4). Consequently, neither does the Gaussian kernel (\ref{gaussian}) on $\bR^d$. The same situation happens to the inverse multiquadric (\ref{multiquadric}) when $\beta=1/2$.

\section{Relaxation of the Admissible Condition (A4)} \label{relaxsec}
\setcounter{equation}{0}

As seen above, the admissible condition (A4) is satisfied for few commonly used kernels. This section aims at weakening this requirement to accommodate more kernels. We are very grateful to the anonymous referee for a useful remark that inspired the approach below.

Let $K$ be a function on $X\times X$ that satisfies (A1)-(A3) and let $\cB$ be constructed by (\ref{cbl1norm}). The condition (A4) is meant to ensure the validity of the linear representer theorem for regularized learning in $\cB$. To see how it can be relaxed, we first examine the role of the linear representer theorem in the learning rate estimate. Consider the $\ell^1$ norm coefficient-based regularization algorithm
\begin{equation}\label{regularization3}
\min_{\vc\in\bC^n}\frac1n\sum_{j=1}^n|K^\bx(x_j)\vc-y_j|^2+\mu \|\vc\|_{\ell^1(\bN_n)}
\end{equation}
where $\bx:=\{x_j:j\in\bN_n\}$ is a sequence of sampling points from the input space $X$, $y_j\in Y\subseteq\bC$ is the observed output on $x_j$, $\mu$ is a positive regularization parameter. Following a commonly used assumption in machine learning, we assume that the sample data $\bz:=\{(x_j,y_j):j\in\bN_n\}\in X\times Y$ is formed by independent and identically distributed instances of a random variable $(x,y)\in X\times Y$ subject to an unknown probability measure $\rho$ on $X\times Y$. Let $c_{\bz,\mu}$ be a minimizer of (\ref{regularization3}). We hope that the obtained function
\begin{equation}\label{fzlambda}
f_{\bz,\mu}(x):=K^\bx(x)\vc_{\bz,\mu},\ \ x\in X
\end{equation}
will well predict the outputs of new inputs from $X$. The performance of a general predictor $f:X\to Y$ is usually measured by
$$
\cE(f):=\int_{X\times Y}|f(x)-y|^2d\rho.
$$
The predictor that minimizes the above error is the regression function
$$
f_\rho(x):=\int_Y yd\rho(y|x),\ \ x\in X,
$$
where $\rho(y|x)$ denotes the conditional probability measure of $y$ with respect to $x$. This optimal predictor $f_\rho$ is unreachable as $\rho$ is unknown. We shall approximate $f_\rho$ with $f_{\bz,\mu}$. More precisely, we expect with a large confidence that the approximation error $\cE(f_{\bz,\mu})-\cE(f_\rho)$ would converge to zero fast as the number of sampling points increases.

A standard approach \cite{CuckerZhou2007} in estimating the error $\cE(f_{\bz,\mu})-\cE(f_\rho)$ is to bound it by the sum of the sampling error, the hypothesis error and the regularization error. Let $g$ be an arbitrary function from $\cB$ and set for each function $f:X\to \bC$
$$
\cE_\bz(f):=\frac1n\sum_{j=1}^n|f(x_j)-y_j|^2.
$$
The approximation error $\cE(f_{\bz,\mu})-\cE(f_\rho)$ can then be decomposed into the sum of four quantities
$$
\cE(f_{\bz,\mu})-\cE(f_\rho)=\cS(\bz,\mu,g)+\cP(\bz,\mu,g)+\cD(\mu,g)-\mu\|f_{\bz,\mu}\|_\cB,
$$
where the {\it sampling error}, the {\it hypothesis error} and the {\it regularization error} are respectively defined by
$$
\begin{array}{ll}
\cS(\bz,\mu,g)&:=\cE(f_{\bz,\mu})-\cE_\bz(f_{\bz,\mu})+\cE_\bz(g)-\cE(g),\\
\cP(\bz,\mu,g)&:=\left(\cE_{\bz}(f_{\bz,\mu})+\mu \|f_{\bz,\mu}\|_\cB\right)-\left(\cE_\bz(g)+\mu \|g\|_\cB\right),\\
\cD(\mu,g)&:=\cE(g)-\cE(f_\rho)+\mu \|g\|_\cB.
\end{array}
$$

Under the condition (A4), $\cB$ satisfies the linear representer theorem. As a result,
\begin{equation}\label{originalrepresenter}
\cE_{\bz}(f_{\bz,\mu})+\mu \|f_{\bz,\mu}\|_\cB=\min_{f\in \spS^{\bx}}\cE_\bz(f)+\mu \|f\|_\cB=\min_{f\in\cB} \cE_\bz(f)+\mu \|f\|_\cB.
\end{equation}
Immediately, one has $\cP(\bz,\mu,g)\le 0$, leading to the estimate
$$
\cE(f_{\bz,\mu})-\cE(f_\rho)\le\cS(\bz,\mu,g)+\cD(\mu,g).
$$
Starting from the above inequality, learning rates of $f_{\bz,\mu}$ can be obtained \cite{SongZhang2011b}. To weaken (A4), we should not stick to the linear representer theorem (\ref{originalrepresenter}). Instead, we wish to replace it with the {\it relaxed linear representer theorem}
\begin{equation}\label{remedyrepresenter}
\min_{f\in \spS^{\bx}}\cE_\bz(f)+\mu \|f\|_\cB\le \min_{f\in\cB} \cE_\bz(f)+\mu\beta_n \|f\|_\cB,
\end{equation}
where $\beta_n$ is a constant depending on the number $n$ of sampling points, the kernel $K$ and the input space $X$. For simplicity, we suppress the notations $K$ and $X$ as they are fixed in our context. The approximation error $\cE(f_{\bz,\mu})-\cE(f_\rho)$ is accordingly factored as
$$
\cE(f_{\bz,\mu})-\cE(f_\rho)=\cS(\bz,\mu,g)+\tilde{\cP}(\bz,\mu,g)+\tilde{\cD}(\mu,g)-\mu\|f_{\bz,\mu}\|_\cB,
$$
where
$$
\begin{array}{rl}
\tilde{\cP}(\bz,\mu,g)&:=\left(\cE_{\bz}(f_{\bz,\mu})+\mu \|f_{\bz,\mu}\|_\cB\right)-\left(\cE_\bz(g)+\mu\beta_n \|g\|_\cB\right),\\
\tilde{\cD}(\mu,g)&:=\cE(g)-\cE(f_\rho)+\mu\beta_n \|g\|_\cB.
\end{array}
$$
By (\ref{remedyrepresenter}), we keep the advantage that $\tilde{\cP}(\bz,\mu,g)\le0$. Therefore,
$$
\cE(f_{\bz,\mu})-\cE(f_\rho)\le\cS(\bz,\mu,g)+\tilde{\cD}(\mu,g).
$$
As long as $\beta_n$ does not increase too fast as $n$ increases, one is still able to obtain a learning rate competitive with those in \cite{SongZhang2011b,Xiao2010}. We shall omit the detailed arguments and assumptions on the kernel $K$, the regression function $f_\rho$ and the input space $X$, as they are similar to those in \cite{SongZhang2011b}. We present one result that for all $0<\delta<1$, there exists a constant $C_\delta$ such that with confidence $1-\delta$, we have
$$
\cE(f_{\bz,\mu})-\cE(f_\rho)\le C_\delta\left((\mu\beta_n)^{\frac{2s}{1+s}}+\frac{\log\frac2\delta}{n}(\mu\beta_n)^{\frac{2s-2}{1+s}}+\frac{\log\frac2\delta}{\sqrt{n}}
(\mu\beta_n)^{\frac{2s-1}{1+s}}
+\frac{\log\frac2\delta+\log(1+n)}{(\mu\beta_n)^2}\beta_n^2n^{-\frac1{1+\theta}}\right),
$$
where $s\in(0,1)$ represents the regularity of $f_\rho$, $\theta>0$ is a positive constant related to assumptions on the kernel $K$ and the input space $X$, \cite{SongZhang2011b}. Thus, as long as $\beta_n^2$ does not cancel the decay of the term $n^{-\frac1{1+\theta}}$, one still has the hope of getting a satisfactory learning rate when $\mu$ is appropriately chosen. We discuss two instances below:
\begin{description}
\item[(i)] If $\beta_n$ is uniformly bounded with a large confidence then $\cE(f_{\bz,\mu})-\cE(f_\rho)$ has the same learning rate as that established in \cite{SongZhang2011b}, that is,
    \begin{equation}\label{learningrate1}
     \cE(f_{\bz,\mu})-\cE(f_\rho)\le C_\delta n^{-\frac{s}{1+2s}\frac1{1+\theta}}\log\frac{2+2n}{\delta}.
    \end{equation}

\item[(ii)] If $\beta_n\le C n^{\alpha}$ for some positive constants $C$ and $\alpha<\frac1{2+2\theta}$ then
\begin{equation}\label{learningrate2}
     \cE(f_{\bz,\mu})-\cE(f_\rho)\le C_\delta n^{-\frac{s}{1+2s}(\frac1{1+\theta}-2\alpha)}\log\frac{2+2n}{\delta}.
    \end{equation}
\end{description}

If we give up the linear representer theorem and pursue the relaxed version (\ref{remedyrepresenter}) instead, how can the admissible condition (A4) be weakened? We next answer this question.

\begin{prop}\label{remedysufficient}
If there exists some $\beta_n \geq 1$ such that for all $\vy\in\bC^n$
\begin{equation}\label{remedysufficientcond}
\min_{f\in\cI_\bx(\vy)}\|f\|_\cB\ge \frac1{\beta_n}\min_{\cI_\bx(\vy)\cap \spS^\bx}\|f\|_\cB
\end{equation}
then the relaxed linear representer theorem (\ref{remedyrepresenter}) holds true for any continuous loss function $V$ and any regularization parameter $\mu$.
\end{prop}
\begin{proof}
Suppose that (\ref{remedysufficientcond}) is satisfied. Let $f_0$ be a minimizer of
$$
\min_{f\in\cB}V(f(\bx))+\lambda\beta_n \|f\|_\cB.
$$
Choose $g$ to be a function in $\spS^\bx$ that interpolates $f_0$ at $\bx$, namely, $g(\bx)=f_0(\bx)$. By (\ref{remedysufficientcond}),
$$
\|g\|_\cB\le \beta_n\|f_0\|_\cB,
$$
which yields
$$
V(g(\bx))+\lambda\|g\|_\cB\le V(f_0(\bx))+\lambda\beta_n\|f_0\|_\cB.
$$
The proof is hence complete.
\end{proof}

We next give a characterization of (\ref{remedysufficientcond}), which gives rise to a relaxation of the admissible condition (A4) and leads to the relaxed linear representer theorem (\ref{remedyrepresenter}).

\begin{thm}\label{remedythm}
Equation (\ref{remedysufficientcond}) holds true for all $\vy\in\bC^n$ if and only if
\begin{equation}\label{relaxation}
\|(K[\bx])^{-1}K_\bx(t)\|_{\ell^1(\bN_n)}\le \beta_n\mbox{ for all }t\in X.
\end{equation}
\end{thm}
\begin{proof}
The set $\cI_{\bx}(\vy)\cap\cS^{\bx}$ consists of only one function $f_0:=K^{\bx}(\cdot)K[\bx]^{-1}\vy$. Let $g$ be an arbitrary function in $\cI_\bx(\vy)\cap\cB_0$. By adding sampling points and assigning the corresponding coefficients to be zero if necessary, we may assume $g\in \cS^{\bx\cup\bt}\cap \cI_\bx(\vy)$ for some $\bt:=\{t_j\in X:j\in\bN_m\}$ disjoint with $\bx$. Let $\vb:=g(\bt)$, and denote by $K[\bt,\bx]$ and $K[\bx,\bt]$ the $n\times m$ and $m\times n$ matrices given by
$$
(K[\bt,\bx])_{jk}:=K(t_k,x_j),\ \ j\in\bN_n,k\in\bN_m, \ \ (K[\bx,\bt])_{jk}:=K(x_k,t_j):\ \ j\in\bN_m,k\in\bN_n.
$$
Then
{\small
\begin{equation}\label{remedythmeq1}
\|g\|_\cB=\left\|\left(\begin{array}{cc}
K[\bx]&K[\bt,\bx]\\
K[\bx,\bt]&K[\bt]
\end{array}\right)^{-1}
\left(\begin{array}{c}
\vy\\ \vb
\end{array}\right)\right\|_{\ell^1(\bN_{n+m})}=\left\|\left(
\begin{array}{c}
K[\bx]^{-1}\vy-K[\bx]^{-1}K[\bt,\bx]\tilde{\vb}\\
\tilde{\vb}
\end{array}\right)\right\|_{\ell^1(\bN_{n+m})},
\end{equation}
}
where
$$
\tilde{\vb}:=(K[\bt]-K[\bx,\bt]K[\bx]^{-1}K[\bt,\bx])^{-1}(\vb-K[\bx,\bt]K[\bx]^{-1}\vy).
$$
Note that as $\vb$ is allowed to equal any vector in $\bC^m$, so is $\tilde{\vb}$.

If (\ref{remedysufficientcond}) holds true for all $\vy\in\bC^n$ then we choose $\bt$ to be a singleton $\{t\}$, $\tilde{b}=1$, and $\vy=K[t,\bx]=K_\bx(t)$ to get
$$
\left\|\left(
\begin{array}{c}
\vzero\\
1
\end{array}\right)\right\|_{\ell^1(\bN_{n+1})}\ge\frac1{\beta_n}\|f_0\|_\cB=\frac1{\beta_n}\left\|K[\bx]^{-1}\vy\right\|_{\ell^1(\bN_n)}=\frac1{\beta_n}\left\|K[\bx]^{-1}K_\bx(t)\right\|_{\ell^1(\bN_n)},
$$
which is (\ref{relaxation}). Conversely, suppose that (\ref{relaxation}) is satisfied. We need to show that for all $g\in\cI_\bx(\vy)$
$$
\|g\|_\cB\ge \frac1{\beta_n}\|f_0\|_\cB=\frac1{\beta_n}\left\|K[\bx]^{-1}\vy\right\|_{\ell^1(\bN_n)}.
$$
We shall discuss the case when $g\in\cI_\bx(\vy)\cap\cB_0$ only as the general case will then follow by the same arguments as those in the last paragraph of the proof of Theorem \ref{repinterp}. Let $g\in\cI_\bx(\vy)\cap\cB_0$ have the norm (\ref{remedythmeq1}). Clearly,
$$
\|g\|_\cB\ge \frac1{\beta_n}\left\|K[\bx]^{-1}\vy\right\|_{\ell^1(\bN_n)}
$$
if $\|K[\bx]^{-1}\vy\|_{\ell^1(\bN_n)}\le\beta_n\|\tilde{\vb}\|_{\ell^1(\bN_m)}$. When $\|K[\bx]^{-1}\vy\|_{\ell^1(\bN_n)}>\beta_n\|\tilde{\vb}\|_{\ell^1(\bN_m)}$, we have
$$
\begin{array}{rl}
\|g\|_\cB&\displaystyle{\ge \|K[\bx]^{-1}\vy\|_{\ell^1(\bN_n)}-\|K[\bx]^{-1}K[\bt,\bx]\tilde{\vb}\|_{\ell^1(\bN_m)}+\|\tilde{\vb}\|_{\ell^1(\bN_m)}}\\
&\displaystyle{\ge \|K[\bx]^{-1}\vy\|_{\ell^1(\bN_n)}-\left(\max_{k\in\bN_m}\|K[\bx]^{-1}K_\bx(t_k)\|_{\ell^1(\bN_n)}\right)\|\tilde{\vb}\|_{\ell^1(\bN_m)}+\|\tilde{\vb}\|_{\ell^1(\bN_m)}}\\
&\displaystyle{\ge \|K[\bx]^{-1}\vy\|_{\ell^1(\bN_n)}-(\beta_n-1)\|\tilde{\vb}\|_{\ell^1(\bN_m)}\ge \|K[\bx]^{-1}\vy\|_{\ell^1(\bN_n)}-(\beta_n-1)\frac1{\beta_n}\|K[\bx]^{-1}\vy\|_{\ell^1(\bN_n)}}\\
&\displaystyle{=\frac1{\beta_n}\|K[\bx]^{-1}\vy\|_{\ell^1(\bN_n)}},
\end{array}
$$
which completes the proof.
\end{proof}

The above result together with the discussion of the application of Proposition \ref{remedysufficient} to regularized learning provides a relaxation of the requirement (A4). The quantity $\sup_{t\in X}\|K[\bx]^{-1}K_\bx(t)\|_{\ell^1(\bN_n)}$ is the Lebesgue constant of the kernel interpolation. Asking it to be exactly bounded by 1 is indeed demanding. Recent numerical experiments \cite{DeMarchi2010} and analysis \cite{Hangelbroek2010} indicate that for many kernels, this Lebesgue constant could be uniformly bounded. In this case, the $\ell^1$-regularized learning in $\cB$ performs well by (\ref{learningrate1}). Furthermore, as long as $\beta_n$ does not increase to infinity too fast, the learning scheme can still work well by (\ref{learningrate2}). Specifically, it was proved in \cite{Hangelbroek2010} that the Lebesgue constant for the reproducing kernel of the Sobolev space on a compact domain is uniformly bounded for quasi-uniform input points (see, Theorem 4.6 therein). Another example is given in \cite{DeMarchi2010} for translation invariant kernels $K(x,y)=\phi(x-y)$, $x,y\in\bR^d$. It was shown there that as long as
\begin{equation}\label{example2cond}
c_1(1+\|\vxi\|_2^2)^{-\tau}\le \hat{\phi}(\xi)\le c_2(1+\|\vxi\|_2^2)^{-\tau},\ \ \|\vxi\|_2>M
\end{equation}
for some positive constants $c_1,c_2,M$ and $\tau$, the Lebesgue constant for quasi-uniform inputs is bounded by a multiple of $\sqrt{n}$. Commonly used kernels satisfying (\ref{example2cond}) include Poisson radial functions \cite{Fornberg2006}, Mat\'{e}rn kernels and Wendland's compactly supported kernels \cite{Wendland}. Finally, we remark from numerical experiments that the following kernels \cite{Schoenberg1938}
$$
\exp\left(-\|x-y\|_{\ell^p(\bN_d)}^\gamma\right), \ \ x,y\in\bR^d,\ \ \gamma\in(0,1),\ \ p=1,2
$$
seem to satisfy (A4) for small enough $\gamma$ and moderate $n$. We shall leave the search of more kernels satisfying (A4) and its relaxation (\ref{relaxation}) as an open question for future study.

\section{Numerical Experiments}

We end this paper with a numerical experiment to show that the regularization algorithm (\ref{regularization1}) is indeed able to yield sparse learning compared to the classical regularization network in machine learning.

We shall use the exponential kernel $K$ (\ref{exponential}). Let $\cB$ be the corresponding RKBS with the $\ell^1$ norm constructed by (\ref{cbl1norm}) and let $\cH_K$ be the RKHS of $K$. We restrict ourselves to the field of real numbers and use the square loss function $V(f(\bx)):= \|f(\bx)-\vy\|^2_{2}$. We shall compare the two models
$$
\min_{f\in\cB}\|f(\bx)-\vy\|^2_{2}+\mu\|f\|_\cB
$$
and
$$
\min_{g\in\cH_K}\|g(\bx)-\vy\|^2_{2}+\mu\|g\|_{\cH_K}^2.
$$
Both of them satisfy the linear representer theorem. Specifically, the minimizers $f_0$ and $g_0$ of the above two models are respectively given by
\begin{equation*}
f_0 = K^{\bx}(\cdot)\vb\mbox{ with }\vb:=\argmin_{\vc\in \bR^n}\{ \|K[\bx]\vc - \vy\|^2_{2} + \mu \|\vc\|_{\ell^1(\bN_n)}\}
\end{equation*}
and
\begin{equation*}
g_0 = K^{\bx}(\cdot)\vh\mbox{ with }\vh:=\argmin_{\vc\in \bR^n}\{ \|K[\bx]\vc - \vy\|^2_{2} + \mu \vc^T K[\bx] \vc\}.
\end{equation*}
We point out that the above $\ell^1$ minimization problem about $\vb$ does not have a closed form solution. There are numerous methods proposed to solve this problem and here we employ the proximity algorithm recently developed in \cite{Micchelli2011}. The closed form of the minimizer $\vh$ is well known to be $(K[\bx] + \mu I_n)^{-1}\vy$. Here $I_n$ denotes the $n\times n$ identity matrix.

For both models, $\bx$ is set to be $200$ equally spaced points in $[-1,1]$ and the output vector $\vy$ is chosen to be the evaluation of the target function
\begin{equation*}
f(x) = \me^{-|x+1|} + \me^{-|x+0.8|} + \me^{-|x|} + \me^{-|x-0.8|} + \me^{-|x-1|},\ \ x\in[-1,1]
\end{equation*}
at $\bx$ and then disturbed by some noise. Also, the regularization parameter $\mu$ for each model will be optimally chosen from $\{10^j: j=-7,-6,\ldots,1\}$ so that the distance between the learned function and the target function in $L^2([-1,1])$ will be minimized. We then compare the approximation accuracy measured by this error and the sparsity for these two models. The sparsity is measured by the number of nonzero components in the coefficient vectors $\vb$ and $\vh$.

\begin{table}[htbp]
\centering
\begin{tabular}{|c|c c|| c c|| c c|}
\hline
&\multicolumn{2}{|c||}{Gaussian noise}& \multicolumn{2}{|c||}{Uniform noise} &\multicolumn{2}{|c|}{Pepper sauce noise}\\
\cline{2-7}
 & Error & Sparsity (Max) & Error & Sparsity (Max)  &Error & Sparsity (Max)\\
\hline
RKHS &2.1E-3 &200 (200) &7.9E-4  &200 (200) &9.4E-4 & 200 (200)\\
RKBS &1.0E-3 & 13.4 (17) &3.6E-4  & 14.7 (25) &4.5E-4 & 14.5 (23)\\
\hline
\end{tabular}
\caption{Comparison of the least square regularization in RKHS and in RKBS with the $\ell^1$ norm for the exponential kernel.}\label{T1}
\end{table}

We test both models with three types of noise: Gaussian noise with variance $0.01$, uniform noise in $[-0.1,0.1]$ and some random pepper sauce noise in $\{-0.1, 0.1\}$. For each type of noise, we run $50$ times of numerical experiments and compute the average approximation error, the average sparsity, and the maximum sparsity in the $50$ experiments. The results are tabulated above.

\bibliographystyle{abbrv}
\bibliography{Banachl1}

\begin{thebibliography}{10}

\bibitem{Argyriou2009}
A.~Argyriou, C.~A. Micchelli, and M.~Pontil.
\newblock When is there a representer theorem? {V}ector versus matrix
  regularizers.
\newblock {\em J. Mach. Learn. Res.}, 10:2507--2529, 2009.

\bibitem{Aronszajn1950}
N.~Aronszajn.
\newblock Theory of reproducing kernels.
\newblock {\em Trans. Amer. Math. Soc.}, 68:337--404, 1950.

\bibitem{BerlinetThomas-Agnan}
A.~Berlinet and C.~Thomas-Agnan.
\newblock {\em Reproducing {K}ernel {H}ilbert {S}paces in {P}robability and
  {S}tatistics}.
\newblock Kluwer, Dordrecht, 2004.

\bibitem{Cand`es2006}
E.~J. Cand{\`e}s, J.~Romberg, and T.~Tao.
\newblock Robust uncertainty principles: exact signal reconstruction from
  highly incomplete frequency information.
\newblock {\em IEEE Trans. Inform. Theory}, 52(2):489--509, 2006.

\bibitem{Chen1998}
S.~S. Chen, D.~L. Donoho, and M.~A. Saunders.
\newblock Atomic decomposition by basis pursuit.
\newblock {\em SIAM J. Sci. Comput.}, 20(1):33--61, 1998.

\bibitem{CuckerSmale2002}
F.~Cucker and S.~Smale.
\newblock On the mathematical foundations of learning.
\newblock {\em Bull. Amer. Math. Soc. (N.S.)}, 39(1):1--49 (electronic), 2002.

\bibitem{CuckerZhou2007}
F.~Cucker and D.-X. Zhou.
\newblock {\em Learning theory: an approximation theory viewpoint}.
\newblock Cambridge Monographs on Applied and Computational Mathematics.
  Cambridge University Press, Cambridge, 2007.
\newblock With a foreword by Stephen Smale.

\bibitem{DeMarchi2010}
S.~De~Marchi and R.~Schaback.
\newblock Stability of kernel-based interpolation.
\newblock {\em Adv. Comput. Math.}, 32(2):155--161, 2010.

\bibitem{Evgeniou2000}
T.~Evgeniou, M.~Pontil, and T.~Poggio.
\newblock Regularization networks and support vector machines.
\newblock {\em Adv. Comput. Math.}, 13(1):1--50, 2000.

\bibitem{Fornberg2006}
B.~Fornberg, E.~Larsson, and G.~Wright.
\newblock A new class of oscillatory radial basis functions.
\newblock {\em Comput. Math. Appl.}, 51(8):1209--1222, 2006.

\bibitem{Giles1967}
J.~R. Giles.
\newblock Classes of semi-inner-product spaces.
\newblock {\em Trans. Amer. Math. Soc.}, 129:436--446, 1967.

\bibitem{Hangelbroek2010}
T.~Hangelbroek, F.~J. Narcowich, and J.~D. Ward.
\newblock Kernel approximation on manifolds {I}: bounding the {L}ebesgue
  constant.
\newblock {\em SIAM J. Math. Anal.}, 42(4):1732--1760, 2010.

\bibitem{James1963/1964}
R.~C. James.
\newblock Characterizations of reflexivity.
\newblock {\em Studia Math.}, 23:205--216, 1963/1964.

\bibitem{Kimeldorf1971}
G.~Kimeldorf and G.~Wahba.
\newblock Some results on {T}chebycheffian spline functions.
\newblock {\em J. Math. Anal. Appl.}, 33:82--95, 1971.

\bibitem{Lumer1961}
G.~Lumer.
\newblock Semi-inner-product spaces.
\newblock {\em Trans. Amer. Math. Soc.}, 100:29--43, 1961.

\bibitem{Micchelli1994}
C.~A. Micchelli and A.~Pinkus.
\newblock Variational problems arising from balancing several error criteria.
\newblock {\em Rendiconti di Matematica, Serie VII}, 14:37--86, 1994.

\bibitem{Micchelli2005a}
C.~A. Micchelli and M.~Pontil.
\newblock On learning vector-valued functions.
\newblock {\em Neural Comput.}, 17(1):177--204, 2005.

\bibitem{Micchelli2011}
C.~A. Micchelli, L.~Shen, and Y.~Xu.
\newblock Proximity algorithms for image models: denoising.
\newblock {\em Inverse Problems}, 27:045009, 2011.

\bibitem{MXZ2006}
C.~A. Micchelli, Y.~Xu, and H.~Zhang.
\newblock Universal kernels.
\newblock {\em J. Mach. Learn. Res.}, 7:2651--2667, 2006.

\bibitem{Schoenberg1938}
I.~J. Schoenberg.
\newblock Metric spaces and positive definite functions.
\newblock {\em Trans. Amer. Math. Soc.}, 44(3):522--536, 1938.

\bibitem{Scholkopf2001}
B.~Sch{\"o}lkopf, R.~Herbrich, and A.~J. Smola.
\newblock A generalized representer theorem.
\newblock In {\em Computational learning theory ({A}msterdam, 2001)}, volume
  2111 of {\em Lecture Notes in Comput. Sci.}, pages 416--426. Springer,
  Berlin, 2001.

\bibitem{Scholkopf2001a}
B.~Sch\"{o}lkopf and A.~J. Smola.
\newblock {\em Learning with Kernels: Support Vector Machines, Regularization,
  Optimization, and Beyond (Adaptive Computation and Machine Learning)}.
\newblock The MIT Press, Cambridge, December 2001.

\bibitem{Shawe-Taylor2004}
J.~Shawe-Taylor and N.~Cristianini.
\newblock {\em Kernel Methods for Pattern Analysis}.
\newblock Cambridge University Press, Cambridge, 2004.

\bibitem{Song2010}
G.~Song and Y.~Xu.
\newblock Approximation of high-dimensional kernel matrices by multilevel
  circulant matrices.
\newblock {\em J. Complexity}, 26(4):375--405, 2010.

\bibitem{SongZhang2011b}
G.~Song and H.~Zhang.
\newblock Reproducing kernel banach spaces with the $\ell^1$ norm ii: error
  analysis for regularized least square regression.
\newblock {\em Neural Comput.}, 23(10):2713--2729, 2011.

\bibitem{Tibshirani1996}
R.~Tibshirani.
\newblock Regression shrinkage and selection via the lasso.
\newblock {\em J. Roy. Statist. Soc. Ser. B}, 58(1):267--288, 1996.

\bibitem{Vapnik1998}
V.~N. Vapnik.
\newblock {\em Statistical Learning Theory}.
\newblock Adaptive and Learning Systems for Signal Processing, Communications,
  and Control. John Wiley \& Sons Inc., New York, 1998.
\newblock A Wiley-Interscience Publication.

\bibitem{Wendland}
H.~Wendland.
\newblock {\em Scattered data approximation}, volume~17 of {\em Cambridge
  Monographs on Applied and Computational Mathematics}.
\newblock Cambridge University Press, Cambridge, 2005.

\bibitem{Wu1995}
Z.~M. Wu.
\newblock Compactly supported positive definite radial functions.
\newblock {\em Adv. Comput. Math.}, 4(3):283--292, 1995.

\bibitem{Xiao2010}
Q.-W. Xiao and D.-X. Zhou.
\newblock Learning by nonsymmetric kernels with data dependent spaces and
  {$\ell^1$}-regularizer.
\newblock {\em Taiwanese J. Math.}, 14(5):1821--1836, 2010.

\bibitem{Zhang2009}
H.~Zhang, Y.~Xu, and J.~Zhang.
\newblock Reproducing kernel {B}anach spaces for machine learning.
\newblock {\em J. Mach. Learn. Res.}, 10:2741--2775, 2009.

\bibitem{Zhangjogo}
H.~Zhang and J.~Zhang.
\newblock Regularized learning in {B}anach spaces as an optimization problem:
  representer theorems.
\newblock {\em J. Global Optim.}
\newblock to appear.

\bibitem{Zhangacha}
H.~Zhang and J.~Zhang.
\newblock Frames, {R}iesz bases, and sampling expansions in {B}anach spaces via
  semi-inner products.
\newblock {\em Appl. Comput. Harmon. Anal.}, 31:1--25, 2011.

\end{thebibliography}

\end{document}